\documentclass{article}

\PassOptionsToPackage{numbers, compress}{natbib}
\usepackage[final]{neurips_2025}

\usepackage[utf8]{inputenc} 
\usepackage[T1]{fontenc}    	
\usepackage{url}            		
\usepackage{booktabs}       	
\usepackage{amsfonts}       	
\usepackage{nicefrac}       	
\usepackage{microtype}      	
\usepackage{xcolor}         		
\usepackage{amsmath}
\usepackage{mathrsfs}
\usepackage{amssymb}
\usepackage{subcaption}
\usepackage{graphicx}
\usepackage{algorithm}
\usepackage{lipsum}
\usepackage{longtable} 
\usepackage{listings}
\usepackage{tcolorbox}
\usepackage{multirow}
\usepackage{natbib}
\usepackage{amsthm}

\usepackage[colorlinks,linkcolor=magenta,citecolor=blue]{hyperref}
\usepackage{amsmath, amssymb}
\usepackage{algorithmicx}
\usepackage{algpseudocode}
\usepackage{mathtools}
\algnewcommand{\LineComment}[1]{\State \(\triangleright\) #1}
    
\newtheorem{theorem}{Theorem}[section]

\newtheorem{lemma}[theorem]{Lemma}
\newtheorem{proposition}[theorem]{Proposition}

\newtheorem{definition}{Definition}[section]

\newtheorem{remark}[theorem]{Remark}
\newtheorem{assumption}[theorem]{Assumption}
\newtheorem*{lemma*}{Lemma}
\newtheorem*{corollary*}{Corollary}
\newtheorem*{theorem*}{Theorem}
\newtheorem*{proposition*}{Proposition}
\newtheorem*{remark*}{Remark}

\def\proofs{\par\noindent{\em Proof sketch.}}

\newcommand{\argmin}{\mathop{\rm argmin}}
\newcommand{\argmax}{\mathop{\rm argmax}}

\newcommand{\ba}{\begin{array}}
\newcommand{\ea}{\end{array}}

\title{Two-Fidelity Best-Action Identification \\ for Stochastic Minimax Tree}

\author{%
  Peter Chen\\
  University of California, Berkeley\\
\texttt{pllc@eecs.berkeley.edu} \\
  \And
  Xi Chen\\
  New York University \\
  \texttt{xc13@stern.nyu.edu} \\
}

\begin{document}

\maketitle

\begin{abstract}
We study fixed-confidence best-action identification (BAI) in stochastic minimax trees. This problem is increasingly relevant in modern AI planning, where deep minimax search and Monte Carlo Tree Search (MCTS) with language model long rollouts face a fundamental tradeoff: heuristic evaluations are cheap but biased, while accurate rollouts are reliable but prohibitively expensive.  We propose 2FFS, a 2-fidelity minimax tree search paradigm that brings multi-fidelity flat bandit ideas into trees. 2FFS combines minimax-style fast expansion with MCTS-style stochastic sampling, adaptively deciding when to exploit cheap biased evaluations and when to invoke expensive accurate evaluations for local certification. We prove fixed-confidence correctness, establish finite stopping for exact identification, and give a polynomial-depth cost upper bound for general-depth trees. Guarantees are further extended to real-world multi-agent reinforcement learning (MARL). Across numerical stochastic-tree experiments, 2FFS uses substantially fewer samples and
computational operations comparing to existing BAI-MCTS baseline\footnote{Numerical implementation: \texttt{\url{https://github.com/PeterLauLukChen/2FFS}}}.
\end{abstract}
\begin{figure*}[h]
\vspace{-1em}
    \centering
    \includegraphics[width=0.95\linewidth]{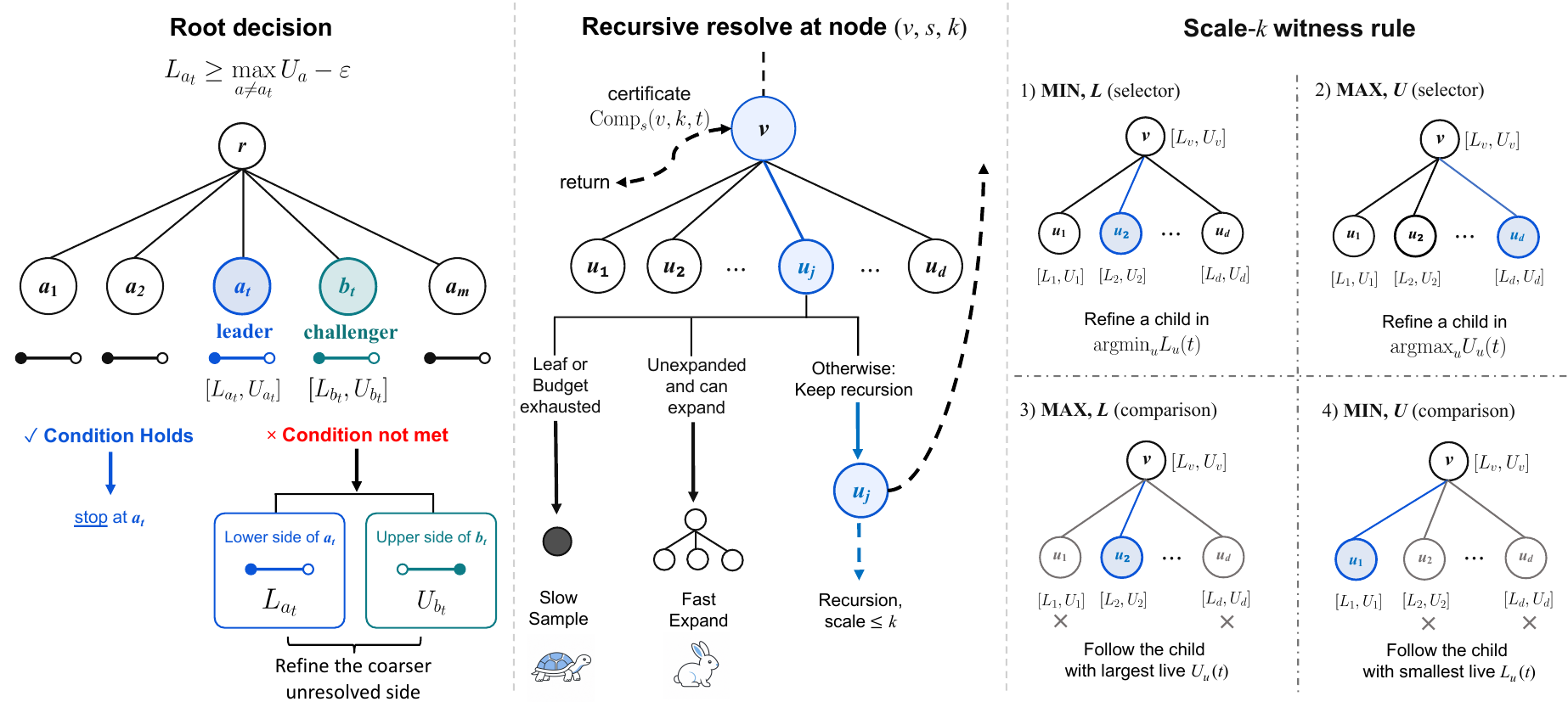}
    \caption{Algorithmic pipeline for 2-fidelity fast-slow (2FFS) minimax tree search.}
    \label{f1}
\end{figure*}

\vspace{-1em}
\section{Introduction}
Modern decision-making systems often face a fundamental budget-allocation tradeoff: when reasoning over a large search space, should computation be used to expand the search tree more deeply, or to improve the reliability of the evaluations assigned to actions already under consideration? In two-player zero-sum planning and multi-agent reinforcement learning, this tradeoff is exemplified by the contrast between minimax-style tree search, which pushes depth using fast heuristic leaf evaluations, and Monte Carlo Tree Search (MCTS), which allocates computation to repeated stochastic sampling and consequently grows a narrower, more selective tree.

A closely related tradeoff also arises in long-horizon reasoning and planning with large language models, where exhaustive tree expansion quickly becomes infeasible and each rollout, especially one involving a long chain-of-thought, can be costly. At a high level, this setting points to a common abstraction: a learner must allocate a limited sampling budget between \emph{cheap but biased} information and \emph{expensive but accurate} information. This viewpoint is closely connected to multi-fidelity bandits, in which an agent may query each arm at different fidelities, trading cost against accuracy in estimating its value. While the multi-fidelity multi-armed bandit (\texttt{MF-MAB}) literature \citep{kandasamy2016multi,poiani2022multifidelity,wang2023multi,poiani2024optimal} formalizes this tradeoff in flat action spaces, and best-action identification in MCTS \citep{kaufmann2017monte} studies fixed-confidence identification in minimax trees under a single stochastic oracle, the corresponding problem of \emph{multi-fidelity best-action identification in minimax trees} remains largely unexplored.

In this work, we formalize this tradeoff in a minimax tree with two complementary sources of information: a \emph{fast} oracle, which is cheap but biased, and a \emph{slow} oracle, which is more costly but statistically accurate. The agent must adaptively allocate computation across the tree, deciding whether to expand the search using fast evaluations or to spend additional budget refining the values of nodes already on the frontier. Our objective is to identify the optimal root action under a fixed-confidence requirement while minimizing total evaluation cost. 
This exposes a principled crossover between $\alpha$-$\beta$-style expansion and MCTS-style sampling: at a frontier node, should the algorithm go one level deeper using the cheap biased fast oracle, or stay and reduce uncertainty using the expensive unbiased slow oracle?

\vspace{-0.5em}
\paragraph{Contributions.}
Our work fills this gap by proposing 2FFS, a two-fidelity search algorithm for fixed-confidence best-action identification in stochastic minimax trees. 2FFS adaptively compares two routes at each decision-critical node: recursively expanding the tree using cheap biased fast evaluations, or locally certifying the node value using expensive but statistically accurate slow evaluations.  On the theoretical side, we prove \((\varepsilon,\delta)\)-PAC correctness and,
for exact identification under a local regularity assumption, establish finite stopping together with a general-depth cost upper bound whose overhead is
polynomial in the tree depth. The same exact-identification upper bound also gives a conservative bound for \(\varepsilon>0\) runs, since the approximate stopping rule can only terminate earlier. To our knowledge, this is the first analytic framework for multi-fidelity bandits in a minimax-tree setting. Empirically, on synthetic stochastic minimax trees, we show that 2FFS substantially improves both sampling and computational efficiency compared with existing baseline.
 
\vspace{-0.5em}
\paragraph{Related works.} Viewed through the compute-allocation lens above, classical perfect-information two-player zero-sum search is dominated by minimax with $\alpha$-$\beta$ pruning, where efficiency is driven by deeper expansion and pruning under deterministic leaf evaluation \cite{Knuth1975alpha,Baudet1978full}. MCTS takes the opposite perspective, treating action selection as an exploration--exploitation problem guided by bandit principles (UCB) and stochastic samples \cite{kocsis2006bandit}. Empirical hybrids such as implicit minimax backups suggest these paradigms are complementary, but also echo classic game-tree pathology results: applying minimax to noisy point estimates can mis-rank moves as depth grows \cite{lanctot2014monte,pearl1983game}. On the theory side, BAI-MCTS formalizes root-action selection as fixed-confidence BAI in minimax trees with an unbiased sampling oracle \cite{kaufmann2017monte}; in parallel, multi-fidelity bandits and Bayesian optimization study cost--accuracy tradeoffs between cheap biased approximations and expensive accurate evaluations \cite{kandasamy2016multi,Kandasamy2016gaussian,kandasamy2017multi,poiani2022multifidelity,wang2023multi,poiani2024optimal}. We defer further related works to Appendix~\ref{app:related-works}.
\section{Problem Setup and Preliminaries}\label{sec:prelim}
\vspace{-0.1em}
In this section, we formalize fixed-confidence BAI in two-fidelity stochastic minimax trees and introduces the notation used by the algorithm and analysis. We first define the tree model, the fast and slow oracles, and the frontier maintained by the learner. Then, we define local and propagated confidence intervals, the effective gaps that determine decision-relevant precision, and the recursive oracle complexity and dyadic scale surrogates used in the upper-bound proof.
\vspace{-0.2em}
\subsection{Minimax tree with two-fidelity oracles}\label{sec:tree}
\vspace{-0.2em}
\paragraph{Stochastic minimax tree.} Let $\mathcal{T}$ be a finite rooted tree with root $r$ and internal nodes alternating between \textsc{Max} and \textsc{Min}, with root $r$ as \textsc{Max} node. For each node $v \in \mathcal{T}$, we denote $\mathrm{Ch}(v)$ as its set of children and $d(v)$ as its depth, and we define \emph{remaining depth} of node $v$ as \( h(v) \coloneqq D - d(v) \), where \(D\) is the depth of \(\mathcal{T}\). For notational simplicity, we assume the tree has balanced depth: a node \(v\) is a leaf if and only if \(h(v)=0\).

Let \(\mathcal{L}\) be the set of leaves, and let each \(\ell\in\mathcal{L}\) have mean payoff \(\mu_\ell\), then its minimax value \(V^*(v)\) is defined by \( V^*(\ell) = \mu_\ell \). For every internal node $v$,
\[
V^*(v)=
\begin{cases}
\max_{u\in \mathrm{Ch}(v)} V^*(u), & \text{if } v \text{ is a \textsc{Max} node},\\[4pt]
\min_{u\in \mathrm{Ch}(v)} V^*(u), & \text{if } v \text{ is a \textsc{Min} node}.
\end{cases}
\]
For each root action $a \in \mathrm{Ch}(r)$, we write \( V^*(r,a) \coloneqq V^*(a) \). Thus, \(V^*(\cdot)\) denotes node values, while \(V^*(r,\cdot)\) denotes root-action values. We assume throughout that \(|\mathrm{Ch}(r)|\ge 2\), so the BAI problem is nontrivial, and that the optimal root action is unique, which is given by \( a^* \in \arg\max_{a \in \mathrm{Ch}(r)} V^*(r,a) \), along with the \emph{root gap} \( \Delta_* \coloneqq V^*(r,a^*) - \max_{a \neq a^*} V^*(r,a) \).

\begin{definition}[Frontier]\label{def:frontier}
A \emph{frontier} $\mathcal{F}\subseteq \mathcal{T}\setminus\{r\}$ is a complete cut of the tree, meaning that every root-to-leaf path intersects $\mathcal{F}$ in exactly one node.
\end{definition}
\vspace{-0.5em}
At each round \(t\), the algorithm has an explored subtree \(\mathcal{T}_t\subseteq\mathcal{T}\) containing the root. A node becomes \emph{exposed} when it first enters \(\mathcal{T}_t\). The current frontier \(\mathcal{F}_t\) consists of the exposed nodes at which recursive expansion has not yet proceeded; it forms a complete cut of the full tree. Expanding an exposed internal node reveals all of its children, queries the fast oracle at those children, and replaces the node by its children in the frontier. Any local slow samples already collected at the expanded node remain available for all future interval computations.
\vspace{-0.5em}
\paragraph{Two-fidelity node-wise evaluations.} Each exposed non-root node \(v\) can be evaluated through two resources. In the multi-fidelity bandit viewpoint \citep{kandasamy2016multi,poiani2022multifidelity,poiani2024optimal}, querying an arm at fidelity \(m\) produces a sample from a distribution \(\nu_{v,m}\) with mean \(\mu_{v,m}\) and known query cost \(\lambda_m\). In our tree setting, the node \(v\) plays the role of the arm, and we use two fidelities. The \emph{fast} oracle \((m=1)\) is cheap and deterministic, but may be biased. The \emph{slow} oracle \((m=2)\) is more expensive and stochastic, but unbiased for the true minimax value \(V^*(v)\). Thus, the fast oracle provides a low-fidelity estimate and the slow oracle provides a high-fidelity estimate of the same node value.

Note that the fast-oracle bias is controlled by the remaining depth $h(v)$ through a known calibration envelope \( B:\{0,\ldots,D\}\to \mathbb{R}_{\ge 0} \), which satisfies
\begin{equation}\label{eq:env}
B(0)=0,
\qquad
B(h+1)\ge B(h)\quad \text{for every } h\in\{0,\ldots,D-1\},
\end{equation}
We normalize query costs as \( \lambda_1=1 \) and \( \lambda_2=c\ge 1 \). The two oracles are defined formally below.

\begin{definition}[Fast oracle $F$]\label{def:fast} The fast oracle returns a deterministic value $V_F(v)$. Equivalently, $\nu_{v,1}$ is a degenerate distribution at $V_F(v)$, so $\mu_{v,1}=V_F(v)$ and \( |V_F(v)-V^*(v)| \leq B(h(v)). \)
\end{definition} 

\begin{definition}[Slow oracle $S$]\label{def:slow} For each node $v$, repeated slow-oracle queries return independent $\sigma$-sub-Gaussian samples $Y_{v,s}$ with distribution $\nu_{v,2}$ and mean \( \mu_{v,2}=V^*(v). \)
\end{definition}

\begin{remark}
The unbiased slow oracle setup is analogous to that of \citet{kaufmann2017monte}, which defines the sampling oracle as mean-correct $\mathbb{E}[Y_{l,s}]=V^*(l)$; in \emph{\S\ref{sec:exp}}, we instantiate it by adding actual noise to true minimax values. Also, unlike classical BAI-MCTS \citep{kaufmann2017monte}, where stochastic samples are attached to leaf rollouts, our slow oracle may be queried at any exposed non-root node. In \emph{\S\ref{sec:worst}}, we add further worst-case guarantee to real-world applications, where slow-oracles are biased.
\end{remark}

\begin{remark}[Fast-oracle envelop]
We note that the envelope structure in Eq.~\eqref{eq:env} can be motivated by finite horizon MARL. Specifically, for a critic $\widehat V_t$ with the correct terminal value, discount $0\le\gamma\le1$, and minimax Bellman residual $\eta_t=\lVert\widehat V_t-\mathcal T_t\widehat V_{t+1}\rVert_\infty$, Bellman nonexpansiveness gives $\lVert\widehat V_t-V_t^*\rVert_{\infty}\le\sum_{j=t}^{H-1}\gamma^{j-t}\eta_j$. A known bound $\eta_j\le\bar\eta$ provides the nondecreasing envelope $B(h)=\bar\eta\sum_{i=0}^{h-1}\gamma^i$.
\end{remark}

\vspace{-0.5em}
\subsection{Intervals, effective gaps, and recursive oracle complexity}
\vspace{-0.15em}
In the fixed-confidence setting, a strategy adaptively chooses exposed nodes and query types, stops at a stopping time \(\tau\), and outputs a recommended action \(\hat a_\tau \in \mathrm{Ch}(r)\). Let \(N^F(\tau)\) denote the total number of fast-oracle queries issued by time \(\tau\), and let \(N_v^S(\tau)\) denote the number of slow-oracle queries issued at node \(v\). The total searching cost is
\begin{equation*}\label{eq:c-tau}
   C_\tau \coloneqq N^F(\tau) + c \sum_{v \in \mathcal{T}} N_v^S(\tau),
\end{equation*}
where $c$ is the slow-oracle cost $\lambda_2$. Given $(\varepsilon,\delta) \!\in\! [0,\infty)\! \times\! (0,1)$, we seek strategies that satisfy the $(\varepsilon,\delta)$-PAC guarantee: \( \mathbb{P}\!\left(\tau<\infty,\ V^*(r,a^*)\! -\! V^*(r,\hat{a}_\tau) \!> \!\varepsilon \right)\!\leq\! \delta \), while seeking low evaluation cost; our cost bounds hold with high probability.

We now introduce the interval objects used by 2FFS and the instance-dependent oracle complexity used in the upper bound. The construction has three layers: (i) local intervals combine the cheap fast evaluation with slow statistical evidence at the same node; (ii) propagated intervals carry child information upward through the minimax recursion; (iii) finally, we establish the ideal amount of oracle complexity needed to certify only the precision that can affect the root decision.
\vspace{-0.5em}
\paragraph{Confidence allocation \& intervals.} We first fix a feasible node-wise confidence allocation \(\boldsymbol{\delta}=(\delta_v)_{v\in\mathcal T\setminus\{r\}}\), which indicates that \( \delta_v\in(0,1)\) and \( \sum_{v\in\mathcal T\setminus\{r\}}\delta_v\le \delta \). For every exposed non-root node \(v\), the fast oracle gives the deterministic interval
\begin{equation}\label{eq:I-fast}
I_v^F=[V_F(v)-B(h(v)),\, V_F(v)+B(h(v))].
\end{equation}
For slow samples, let \(\widehat V_{v,n}\) be the empirical mean of the first \(n\) slow-oracle samples at \(v\). Following from fixed-confidence bandit literature \citep{garivier2016optimal,kaufmann2017monte,howard2021time,kaufmann2021mixture,poiani2022multifidelity}, we use the time-uniform radius \(\beta_v(\cdot,\delta_v)\) satisfying
\begin{equation}\label{eq:conf-radius}
\mathbb{P}\left(
\exists\, n\ge 1:\,
\left|\widehat V_{v,n}-V^*(v)\right|>\beta_v(n,\delta_v)
\right)\le \delta_v,
\end{equation}
where \(n\mapsto \beta_v(n,\delta_v)\) is non-increasing and \(\beta_v(n,\delta_v)\to 0\) as \(n\to\infty\). Let \(N_v^S(t)\) be the number of slow queries issued at \(v\) by round \(t\), the slow interval is then the running intersection
\begin{equation}\label{eq:slow-interval}
I_v^S(t)=
\begin{cases}
\displaystyle\bigcap_{j=1}^{N_v^S(t)}
\left[
\widehat V_{v,j}-\beta_v(j,\delta_v),\,
\widehat V_{v,j}+\beta_v(j,\delta_v)
\right], & N_v^S(t)\ge 1,\\[10pt]
(-\infty,+\infty), & N_v^S(t)=0,
\end{cases}
\end{equation}
which further gives the direct local interval \( I_v^{\mathrm{loc}}(t)\coloneqq I_v^F\cap I_v^S(t) \). Finally, we define the simultaneous-validity event by Eq.~\eqref{eq:E-delta}. Also, by Eq.~\eqref{eq:conf-radius} and the union bound, we always have \(\mathbb P(\mathcal E_\delta)\ge 1-\delta\).
\begin{equation}\label{eq:E-delta}
\mathcal E_\delta
\coloneqq
\bigcap_{v\in\mathcal T\setminus\{r\}}
\bigcap_{n\ge 1}
\left\{
\left|\widehat V_{v,n}-V^*(v)\right|
\le
\beta_v(n,\delta_v)
\right\}.
\end{equation}
\vspace{-1.5em}
\paragraph{Propagated intervals.} If an internal node \(v\) has been expanded, its children are exposed and carry effective intervals \(I_u(t)=[L_u(t),U_u(t)]\). The child-backup interval is
\begin{equation}\label{eq:minmax-propegate}
I_v^{\mathrm{ch}}(t)=
\begin{cases}
\left[
\max_{u\in \mathrm{Ch}(v)} L_u(t),\,
\max_{u\in \mathrm{Ch}(v)} U_u(t)
\right], & \text{if } v \text{ is a \textsc{Max} node},\\[8pt]
\left[
\min_{u\in \mathrm{Ch}(v)} L_u(t),\,
\min_{u\in \mathrm{Ch}(v)} U_u(t)
\right], & \text{if } v \text{ is a \textsc{Min} node}.
\end{cases}
\end{equation}
If \(v\) is a leaf or not expanded yet, we set \(I_v^{\mathrm{ch}}(t)=(-\infty,+\infty)\). For every explored non-root node, the effective interval combines direct local evidence and descendant evidence: \( I_v(t)=I_v^{\mathrm{loc}}(t)\cap I_v^{\mathrm{ch}}(t) \). We note that the minimax backup in Eq.~\eqref{eq:minmax-propegate} preserves validity and does not increase interval width. The following two lemmas formalize these properties; their proofs are given in Appendix~\ref{app:prelim-proofs}.

\begin{lemma}\label{lem:minimax-propagation}
Fix a frontier \(\mathcal{F}\) and suppose that each frontier interval \([L_v,U_v]\), \(v\in\mathcal{F}\), contains \(V^*(v)\). Under the recursive backup in Eq.~\eqref{eq:minmax-propegate}, \emph{(i)} every ancestor \(u\) of a frontier node satisfies \(V^*(u)\in[L_u,U_u]\); and \emph{(ii)} if every frontier interval has half-width at most \(w\), then every propagated ancestor interval also has half-width at most \(w\).
\end{lemma}

\begin{lemma}[Simultaneous interval validity]\label{lem:interval-validity}
On \(\mathcal E_\delta\), every effective interval constructed from Eqs.~\eqref{eq:I-fast}--\eqref{eq:minmax-propegate} is valid at every round \(t\). In particular, \(V^*(v)\in I_v(t)\) for every explored non-root node \(v\), and the intervals of root children and the root backup are valid as well.
\end{lemma}
\vspace{-0.5em}
\paragraph{Effective gap.} A node does not necessarily need to be estimated down to its own local sibling gap. The value of a node \(v\) can affect the final recommendation only if its uncertainty propagates through every ancestor and changes the comparison between root actions. Along this path, any large value separation creates a bottleneck: once the uncertainty below \(v\) is smaller than that separation, further refinement cannot change the root decision. We therefore define the \emph{effective gap} of a node as the precision scale needed to certify that node for the purpose of root-action identification.

\begin{definition}[Effective gap]\label{def:eff-gap}
Set the root effective gap to \(\Delta_r^{\mathrm{eff}}\coloneqq \Delta_*\). For every non-root node \(v\), let \(p(v)\) denotes the parent of \(v\), we define recursively
\[
\Delta_v^{\mathrm{eff}}
\coloneqq
\max\!\left\{
\Delta_{p(v)}^{\mathrm{eff}},\,
|V^*(v)-V^*(p(v))|
\right\}.
\]
\end{definition}
\vspace{-0.5em}
Thus, \(\Delta_v^{\mathrm{eff}}\) records the largest bottleneck encountered on the path from the root to \(v\). If this quantity is large, then \(v\) only needs to be estimated coarsely: Proposition~\ref{prop:active-call-localization} shows that a call at scale $k$ requires $\Delta_v^{\mathrm{eff}}\le2\rho_k$.
\vspace{-0.5em}
\paragraph{Reference recursive oracle complexity.} To derive cost guarantees, we first define an idealized reference oracle complexity \(H^*\) under perfect gap information within our recursive certification model. Our main theorem shows that 2FFS, despite not knowing the effective gaps or the cheaper route in advance, has cost within a polynomial-depth factor of this ideal reference complexity.

For a target precision \(\rho>0\), we first define the local cost of certifying a single exposed node \(v\). Let
\[
m_v(\rho,\delta_v)
\coloneqq
\inf\!\left\{
n\in\mathbb N:\,n\ge 1,\,
\beta_v(n,\delta_v)\le \tfrac14 \rho
\right\}.
\]
After \(m_v(\rho,\delta_v)\) slow samples, the slow confidence interval has width at most \(\rho/2\). The factor \(1/4\) is chosen so that either the fast interval alone, \(2B(h(v))\le \rho/2\), or the slow confidence interval alone, \(2\beta_v(n,\delta_v)\le \rho/2\), is sufficient to make the direct local interval \(I_v^{\mathrm{loc}}\) have width at most \(\rho/2\).

The corresponding local certification cost is
\begin{equation}\label{eq:gamma-rho}
\Gamma_v(\rho,\delta_v)=
\begin{cases}
0, & \text{if } B(h(v))\le \tfrac14 \rho,\\[3pt]
c\cdot m_v(\rho,\delta_v), & \text{if } B(h(v))> \tfrac14 \rho.
\end{cases}
\end{equation}
Thus, \(\Gamma_v(\rho,\delta_v)\) is the slow-oracle cost needed to certify the value of \(v\) locally at precision \(\rho\), after \(v\) has already been exposed. The first case means that the fast interval is already narrow enough, so no slow sample is needed at \(v\).

We now define the recursive subtree oracle complexity. For a fixed confidence allocation \(\boldsymbol{\delta}\), \(J_v^*(\boldsymbol{\delta})\) is a reference cost for the subtree rooted at \(v\), using effective-gap information. At each internal node, it compares two routes: certify \(v\) locally, or expand \(v\) and add all child reference costs:
\begin{equation}\label{eq:J-star}
J_v^*(\boldsymbol{\delta})=
\begin{cases}
\Gamma_v(\Delta_v^{\mathrm{eff}},\delta_v), & \text{if } h(v)=0,\\[5pt]
\min\!\left\{
\Gamma_v(\Delta_v^{\mathrm{eff}},\delta_v),\,
|\mathrm{Ch}(v)|+\sum_{u\in \mathrm{Ch}(v)} J_u^*(\boldsymbol{\delta})
\right\}, & \text{if } h(v)\ge 1.
\end{cases}
\end{equation}
This induces the root oracle complexity \( H(\boldsymbol{\delta}) \coloneqq |\mathrm{Ch}(r)|+\sum_{a\in \mathrm{Ch}(r)} J_a^*(\boldsymbol{\delta}) \), where \(|\mathrm{Ch}(r)|\) is the cost of initially exposing all root actions, and gives us the optimized recursive oracle complexity $H^*$:
\begin{equation}\label{eq:H-star}
H^*
\coloneqq
\inf_{\boldsymbol{\delta}\,\mathrm{feasible}}
H(\boldsymbol{\delta}).
\end{equation}

\begin{remark}\label{rmk:H^*}
We note that instance-dependent complexity measures are standard in fixed-confidence BAI. Classical bounds are often stated in terms of gap-dependent hardness quantities such as \(H_1=\sum_{i\neq i^*}\Delta_i^{-2}\), while refined analyses use information-theoretic characteristic times to quantify the intrinsic difficulty of an instance \citep{audibert2010best,jamieson2014lilucb,kaufmann2016complexity}. BAI-MCTS and multi-fidelity BAI follow the same convention, using tree- or fidelity-dependent oracle complexities to state adaptive cost guarantees \citep{kaufmann2017monte,poiani2022multifidelity,poiani2024optimal}. Our \(H^*\) serves as an analysis-only recursive oracle complexity under perfect gap information, rather than an information-theoretic lower bound; see Appendix~\ref{app:theory-qualifications}.
\end{remark}
\vspace{-0.5em}
\paragraph{Dyadic scale surrogates.} Following from Remark~\ref{rmk:H^*}, the algorithm does not know the effective gaps and therefore cannot target \(J_v^*(\boldsymbol{\delta})\) directly. Instead, 2FFS works on a dyadic precision grid, which initializes by exposing the root children with fast queries, then sets
\[
\rho_0\coloneqq \max_{a\in \mathrm{Ch}(r)} \bigl(U_a(0)-L_a(0)\bigr),
\qquad
\rho_{k+1}\coloneqq \tfrac12 \rho_k,
\qquad k\ge 0.
\]
If \(\rho_0=0\), the root-child values are already known exactly and the algorithm stops immediately; otherwise the following scales are used. A node \(v\) is \emph{scale-\(k\) safe} at round \(t\) if \( U_v(t)-L_v(t)\le \tfrac12 \rho_k \). The scale-\(k\) local certification cost is \( \Gamma_v^{(k)}(\delta_v)\coloneqq \Gamma_v(\rho_k,\delta_v) \), and the operational rules that decide which scale and which node to refine are given in Algorithm~\ref{alg:2ffs-t}.

\section{2FFS Tree Search: Two-Fidelity Fast-Slow Tree Search}\label{sec:fs2fa}
\vspace{-0.5em}
2FFS maintains a confidence interval \([L_v(t),U_v(t)]\) for each exposed node and repeatedly certifies the root decision. At each round, the algorithm first checks whether some root action is already \(\varepsilon\)-optimal according to the current intervals. If not, it selects one decision-relevant side of one root child and calls a recursive resolver \(\textsc{Resolve}(v,s,k)\), where \(s\in\{L,U\}\) is the endpoint to certify and \(k\) is the active dyadic scale. Figure~\ref{f1} visualizes these three pieces: the root test, the recursive resolver, and the scale-\(k\) child-selection rules.

At the root, let \( a_t\in \argmax_{a\in \mathrm{Ch}(r)}L_a(t)\) and \( b_t\in \argmax_{a\in\mathrm{Ch}(r)\setminus\{a_t\}}U_a(t) \) be the current leader and challenger. 2FFS stops once a root action is certified to be \(\varepsilon\)-optimal: \( L_{\hat a_t}(t)\!\ge\! \max_{a\neq \hat a_t}U_a(t)-\varepsilon \). If this test fails, the leader's lower endpoint and the challenger's unresolved endpoint information are the two quantities that can still change the root decision. 2FFS refines whichever of these two root-side certificates is currently coarser. For the leader this means the lower side \(L_{a_t}\); for the challenger it means the coarsest unresolved side among \(L_{b_t}\) and \(U_{b_t}\).

The call \(\textsc{Resolve}(v,s,k)\) tries to certify endpoint \(s\) of node \(v\) at scale \(\rho_k\). It has two competing routes. The local route takes slow samples directly at \(v\). The recursive route expands below \(v\), uses fast evaluations on the children, and continues recursively only along the currently relevant child. These routes remain available throughout the run: even after \(v\) has been expanded, 2FFS may still take slow samples at \(v\). This \emph{local reversibility} is important because the cheaper route is not known in advance. Recursive spending at node \(v\) and scale \(k\) is capped by \( \mathsf B_{v,k}^{\mathrm{rec}} = \alpha_{h(v)}\Gamma_v^{(k)}(\delta_v) \), and every child call uses a scale no finer than the parent scale \(k\). If the recursive route cannot make admissible progress within this budget, the resolver falls back to one slow local query at \(v\).

\begin{algorithm}[t]
\caption{\textsc{2FFS} algorithm skeleton}
\label{alg:2ffs-t}
\begin{algorithmic}[1]
\State \textbf{Input:} confidence schedule \((\delta_v)\), accuracy \(\varepsilon\), depth budgets \((\alpha_h)_{h=0}^D\)
\State Initialize the root children with fast-oracle intervals and propagate intervals to the root
\While{no root child \(\hat a\) satisfies \(L_{\hat a}\ge \max_{a\neq \hat a}U_a-\varepsilon\)}
    \State Let \(a_t\gets\argmax_a L_a(t)\) be the leader and \(b_t\gets\argmax_{a\neq a_t}U_a(t)\) be the challenger
    \State Choose the coarser unresolved root side: \(L_{a_t}\) or the challenger's coarsest unresolved side
    \LineComment{\textcolor{blue}{\(x\) is the selected node, \(s\) the selected side, and \(k\) the active scale.}}
    \State \Call{Resolve}{$x,s,k$} 
\EndWhile
\State \textbf{return} any separated root child
    
\Procedure{Resolve}{$v,s,k$}
    \If{the side certificate \(\mathsf{Comp}_s(v,k,t)\) already holds}
        \State \textbf{return}
    \EndIf
    \If{\(v\) is a leaf, its recursive budget is exhausted, or it is unexpanded and expansion is unaffordable}
        \State Take one slow local sample at \(v\) and update affected intervals and certificates
    \ElsIf{\(v\) is unexpanded and expansion fits the remaining recursive budget}
        \State Expand \(v\), query the fast oracle at its children, and update affected intervals and certificates
    \Else
        \State Select the current live child using selector-comparison rule in Figure~\ref{f1} (right).
        \State Recursively resolve the selected child at an unresolved scale no larger than \(k\)
        \If{the child call is blocked only by the inherited parent cap}
            \State Return the blocked status upward
        \ElsIf{the child call is blocked by \(v\)'s own recursive budget}
            \State Take one slow local sample at \(v\)
        \EndIf
    \EndIf
\EndProcedure
\end{algorithmic}
\end{algorithm}

Inside an expanded node, the recursive route follows the live child rule in Figure~\ref{f1} (right). In selector cases, the live child is the endpoint that currently determines the minimax backup. In comparison cases, 2FFS does not force all children to become fully certified. Instead, it lazily removes a child once the child is already outside the scale-\(k\) comparison margin, or once the child has the required same-side certificate. Among the remaining live children, only the current endpoint-most blocker is refined recursively. The comparison rows use the following discharge margins. For a \textsc{Max} node on side \(L\), write \(\lambda_L(v,t)=\max_u L_u(t)\); a child is no longer live if
\[
U_u(t)\le \lambda_L(v,t)+\rho_k/2
\quad\text{or}\quad
\mathsf{Comp}_L(u,k,t).
\]
For a \textsc{Min} node on side \(U\), write \(\lambda_U(v,t)=\min_u U_u(t)\); a child is no longer live if
\[
L_u(t)\ge \lambda_U(v,t)-\rho_k/2
\quad\text{or}\quad
\mathsf{Comp}_U(u,k,t).
\]
The formal implementation maintains scale-local counters, monotone side flags \(\mathsf{Done}_s(v,k)\), observable certificates \(\mathsf{Comp}_s(v,k,t)\), and capped unresolved scales \(K_s^{\le k}(v,t)\). These details are given in Appendix~\ref{app:alg-details}, which are not needed to understand the main flow above but important for the proofs.

\vspace{-0.25em}
\subsection{Theoretical results}
\vspace{-0.25em}
We present three results: algorithmic correctness and a general-depth cost upper bound guarantee along with finite stopping time guarantee. First, we provide algorithm's $(\varepsilon,\delta)$-PAC correctness:

\begin{theorem}[PAC correctness]\label{thm:pac-correct}
Fix any feasible confidence allocation \(\boldsymbol{\delta}=(\delta_v)_{v\in\mathcal T\setminus\{r\}}\) and any \(\varepsilon\ge 0\). Consider \emph{Algorithm~\ref{alg:2ffs-t}} run with accuracy \(\varepsilon\), let \(\tau\) be its stopping time, and, on the event \(\{\tau<\infty\}\), let \(\hat a_\tau\) be the returned root child. Then, \( \mathbb P\!\left( \tau<\infty \text{ \emph{and} } V^*(r,a^*)-V^*(r,\hat a_\tau)>\varepsilon \right) \le \delta \). In fact, on the simultaneous-validity event \(\mathcal E_\delta\), every finite-time recommendation made by the stopping rule is \(\varepsilon\)-optimal.

\end{theorem}
\begin{remark}
We note that \emph{Theorem~\ref{thm:pac-correct}} uses only simultaneous interval validity and the root stopping rule and does not use any assumptions in the later upper-bound proof. Its proof is very straightforward, and we defer it to \emph{Appendix~\ref{app:pac-correct}.} The finite stopping guarantee is elaborated in \emph{Theorem~\ref{thm:budgeted-general-depth}}.
\end{remark}
\vspace{-0.5em}
Now, we move to establish the upper bound for cost using the reference oracle complexity $H^*$. Before that, we first explicitly state several required regular assumptions along with their discussions.

\begin{assumption}[Local dyadic scale regularity]\label{ass:local-scale-regularity}
Fix the confidence allocation \(\boldsymbol{\delta}\), we assume there exist constants \(\Lambda_{\mathrm{pre}}\ge 1\) and \(\Lambda_{\mathrm{gap}}\ge 1\) such that, for every non-root node \(v\) and every \(K\ge 0\),
\begin{equation}\label{eq:local-prefix-regularity}
\sum_{j=0}^{K}\Gamma_v^{(j)}(\delta_v)
\le
\Lambda_{\mathrm{pre}}\Gamma_v^{(K)}(\delta_v),
\end{equation}
and, for every non-root node \(v\),
\begin{equation}\label{eq:local-gap-regularity}
\Gamma_v\!\left(\frac{\Delta_v^{\mathrm{eff}}}{2},\delta_v\right)
\le
\Lambda_{\mathrm{gap}}\,
\Gamma_v(\Delta_v^{\mathrm{eff}},\delta_v).
\end{equation}
We write \( \Lambda_{\mathrm{loc}} \coloneqq \Lambda_{\mathrm{pre}}\Lambda_{\mathrm{gap}} \).
\end{assumption}
\begin{remark}
\emph{Assumption~\ref{ass:local-scale-regularity}} is purely local: it does not compare the raw recursive dyadic DP with \(J^*\). The prefix condition is the standard dyadic regularity of a statistical sample-complexity curve; it holds with a constant for polynomial laws, e.g., \(m_v(\rho,\delta_v)\asymp \rho^{-2}\log(1/\delta_v)\) with uniform comparison constants over the positive-cost scales. The second condition only asks that moving from the effective gap to the neighboring dyadic half-scale changes the local stop cost by at most a constant factor. This is the only place where the deterministic fast-bias cutoff in \emph{Eq.~\eqref{eq:gamma-rho}} can create a knife-edge, which is ruled out by the assumption at the exact effective gap.
\end{remark}

\begin{lemma}\label{lem:depth-polynomial-overhead}
For a fixed nondecreasing depth-budget sequence \((\alpha_h)_{h=0}^{D}\) with $\alpha_h\ge1$, define \(P_0=\Lambda_{\mathrm{loc}}\) and, for \(h\ge 1\),
\begin{equation}\label{eq:P-depth-recursion}
P_h
\coloneqq
\max\!\left\{
(1+\alpha_h)\Lambda_{\mathrm{loc}},\,
\left(1+\frac{\Lambda_{\mathrm{pre}}}{\alpha_h}\right)P_{h-1}
\right\}.
\end{equation}
With the concrete choice \(\alpha_h=(h+1)^2\), this recursion gives \(P_D=\mathcal{O}_{\Lambda_{\mathrm{pre}},\Lambda_{\mathrm{gap}}}(D^2)\).
\end{lemma}
\vspace{-0.5em}
\begin{proof}
We denote \(A=\Lambda_{\mathrm{loc}}\) and \(B=\Lambda_{\mathrm{pre}}\), and unrolling the recursion shows that \(P_D\) is at most the largest local term \((1+\alpha_i)A\), multiplied by the product of later recursive multipliers:
\[
P_D
\le
A
\max_{0\le i\le D}
\left[
\bigl(1+\alpha_i\bigr)
\prod_{j=i+1}^{D}\left(1+\frac{B}{\alpha_j}\right)
\right].
\]
For \(\alpha_j=(j+1)^2\), the product is bounded by
\[
\prod_{j=1}^{D}\left(1+\frac{B}{(j+1)^2}\right)
\le
\exp\!\left(B\sum_{j=1}^{\infty}\frac{1}{(j+1)^2}\right),
\]
which is a constant depending only on \(B\). Since \(\max_{i\le D}(1+\alpha_i)=O(D^2)\), we obtain \(P_D=O_{\Lambda_{\mathrm{pre}},\Lambda_{\mathrm{gap}}}(D^2)\), which is the key ingredient for the following general-depth cost bound.
\end{proof}

\begin{theorem}[General-depth upper bound]\label{thm:budgeted-general-depth}
Fix any feasible confidence allocation \(\boldsymbol{\delta}=(\delta_v)_{v\in\mathcal T\setminus\{r\}}\) satisfying \emph{Assumption~\ref{ass:local-scale-regularity}}, and run \emph{Algorithm~\ref{alg:2ffs-t}} with \(\varepsilon=0\). On \(\mathcal E_\delta\), the algorithm is guaranteed to stop at a finite time \(\tau\), returns \(\hat a_\tau=a^*\), and satisfies \( C_\tau \le P_D\,H(\boldsymbol{\delta}) \).

With the depth budgets \(\alpha_h=(h+1)^2\), the cost \( C_\tau \le \mathcal{O}_{\Lambda_{\mathrm{pre}},\Lambda_{\mathrm{gap}}} \!\left(D^2H(\boldsymbol{\delta})\right) \). Consequently, suppose the regularity constants \(\Lambda_{\mathrm{pre}}\) and \(\Lambda_{\mathrm{gap}}\) are uniform along a near-optimal feasible sequence: for every \(\eta>0\), there exists a feasible allocation \(\boldsymbol{\delta}^{\eta}\) satisfying \emph{Assumption~\ref{ass:local-scale-regularity}} with these same constants and
\[
H(\boldsymbol{\delta}^{\eta})\le H^*+\eta.
\]
Then, the corresponding run satisfies, with probability at least \(1-\delta\),
\[
\hat a_\tau=a^*,
\qquad
C_\tau
\le
P_D\bigl(H^*+\eta\bigr).
\]
In particular, when \(\alpha_h=(h+1)^2\) and the local regularity constants are uniform constants,
\[
C_\tau
\le
\mathcal{O}_{\Lambda_{\mathrm{pre}},\Lambda_{\mathrm{gap}}}
\!\left(D^2(H^*+\eta)\right),
\]
and, since \(H^*\ge |\mathrm{Ch}(r)|\ge 2\), one obtains \(\mathcal{O}_{\Lambda_{\mathrm{pre}},\Lambda_{\mathrm{gap}}}(D^2H^*)\) by taking \(\eta\le H^*\), or \(\eta=0\) whenever the infimum in \emph{Eq.~\eqref{eq:H-star}} is attained with the same regularity constants.
\end{theorem}

\proofs \ \
To our knowledge, this is the first theoretical result to establish a bound for multi-fidelity bandit in the tree setting. We therefore hope the proof structure sets up a paradigm for future work, and Theorem~\ref{thm:budgeted-general-depth} is proved formally in Appendix~\ref{app:pf3.6}. Overall, we have one main goal: charge the actual oracle calls made by 2FFS to the ideal recursive oracle complexity \(H(\boldsymbol{\delta})\). First, every certificate used by 2FFS is shown to be sound: once a side of a node is certified at scale \(k\), the corresponding endpoint error is at most \(\rho_k/2\). Second, a recursive call is shown to spend effort only on children that can still affect the parent certificate at the current scale; since child calls are capped by the parent scale, the algorithm avoids irrelevant fine-scale work below already separated parts of the tree. Third, for each node and scale, the budgeted resolver spends at most a local certification budget plus its recursive budget, and the same node-side-scale is never charged again after certification. Summing these charges over active scales, while reserving the cap needed for the next step, gives a subtree bound with overhead \(P_{h(v)}\). Lemma~\ref{lem:depth-polynomial-overhead} then shows that this overhead is only quadratic in depth for \(\alpha_h=(h+1)^2\). Applying the subtree bound to the root children gives a finite global cost bound; since every non-stopping round makes positive progress, the algorithm must stop, and interval validity gives correctness.

\subsection{Guarantees for real-world multi-agent RL application}\label{sec:worst}
In many real-world applications, e.g., multi-agent RL (MARL) with minimax trees, slow oracles may not provide perfectly unbiased estimates of the minimax value $V^*$, and a tight, accurate fast-oracle envelope $B$ may be unavailable and must instead be approximated by a proxy. In this section, we establish additional theoretical guarantees for such imperfect real-world settings.

\vspace{-0.5em}
\paragraph{Biased slow oracles.} We show that pointwise slow-oracle bias can be handled without accumulation across tree depths. Specifically, Theorem~\ref{thm:biased-slow} gives a correctness guarantee using a known bias bound.

\begin{theorem}\label{thm:biased-slow}
Suppose the slow oracle has fixed mean $V^*(v)+b_v$, with known $\xi\ge\max_v|b_v|$, and its samples satisfy Eq.~\eqref{eq:conf-radius} around this mean. Replacing each slow radius $\beta_v$ by $\beta_v+\xi$ preserves the $(\varepsilon,\delta)$ finite-output correctness of 2FFS. The bias inflation does not accumulate with tree depth.
\end{theorem}

\paragraph{Loose fast-oracle envelope.} In practical MARL applications, the true minimax value is unavailable, and thus a tight fast-oracle envelope cannot be constructed. Instead, one can use a calibrated proxy envelope that shrinks with the remaining horizon, for example,
\[
\hat B(h)=1+\frac{h-1}{D-1},
\qquad \hat B(0)=0,
\]
where $D\ge2$ is the full tree depth and $h\ge1$ is the remaining depth, provided $\hat B(h)\ge B(h)$. We show that such a loose envelope preserves the PAC guarantee.

\begin{theorem}\label{thm:loose-envelope}
Let \(B(h)\) denote the true fast-oracle envelope, and let \(\hat B(h)\) be a nondecreasing proxy envelope with $\hat B(0)=0$ and $\hat B(h)\ge B(h)$ for all \(h\). Replacing \(B(h)\) with \(\hat B(h)\) preserves the finite-output PAC guarantee and, on $\mathcal E_\delta$ for the same query history, can only widen the resulting confidence intervals.
\end{theorem}
The proofs and qualifications for both results are given in Appendix~\ref{app:robustness-proofs}.

\section{Numerical Experiments}\label{sec:exp}
\vspace{-1.5em}
\begin{table}[h]
\centering
\caption{Results over stochastic minimax trees, with mean $\pm$ standard deviation reported.}
\resizebox{\linewidth}{!}{
\begin{tabular}{l|cccc}
\toprule
Method &Stopping $(\uparrow)$ & Accuracy $(\uparrow)$ & Sampling count $(\downarrow)$ & Operation count $(\downarrow)$ \\
\midrule
(i) $D\!=\!5,b=8$; $37,449$ nodes\\
\midrule
\texttt{2FFS}
& 1.00 & 1.00 
& \(5.39{\times}10^3 \pm 1.27{\times}10^3\)
& \(8.38{\times}10^7 \pm 4.05{\times}10^7\) \\

\texttt{BAI-MCTS}
& 1.00 & 0.99
& \(8.80{\times}10^5 \pm 3.39{\times}10^5\)
& \(2.38{\times}10^8 \pm 9.49{\times}10^7\) \\

\texttt{Minimax-fast}
& 1.00 & 0.91
& \(1.49{\times}10^4 \pm 3.35{\times}10^3\)
& \(8.20{\times}10^5 \pm 1.91{\times}10^5\) \\

\texttt{Slow-only}
& 0.47 & 0.47
& \(5.87{\times}10^5 \pm 4.86{\times}10^5\)
& \(8.62{\times}10^7 \pm 7.15{\times}10^7\) \\
\midrule
(ii) $D\!=\!7,b=6$; $335,923$ nodes\\
\midrule
\texttt{2FFS}
& 1.00 & 1.00
& \(1.77{\times}10^4 \pm 2.64{\times}10^3\)
& \(1.06{\times}10^9 \pm 3.02{\times}10^8\) \\

\texttt{BAI-MCTS}
& 0.98 & 0.98
& \(1.75{\times}10^7 \pm 6.71{\times}10^6\)
& \(4.85{\times}10^9 \pm 1.87{\times}10^9\) \\

\texttt{Minimax-fast}
& 1.00 & 0.88
& \(5.52{\times}10^4 \pm 7.08{\times}10^3\)
& \(3.48{\times}10^6 \pm 4.59{\times}10^5\) \\

\texttt{Slow-only}
& 0.73 & 0.70
& \(4.22{\times}10^6 \pm 9.16{\times}10^6\)
& \(1.49{\times}10^9 \pm 4.13{\times}10^8\) \\
\midrule
(iii) $D\!=\!10,b=3$; $88,573$ nodes\\
\midrule
\texttt{2FFS}
& 1.00 & 1.00 
& \(1.31{\times}10^4 \pm 2.34{\times}10^3\)
& \(2.49{\times}10^9 \pm 8.64{\times}10^8\) \\

\texttt{BAI-MCTS}
& 0.98 & 0.98
& \(1.91{\times}10^7 \pm 4.19{\times}10^6\)
& \(4.62{\times}10^9 \pm 9.34{\times}10^8\) \\

\texttt{Minimax-fast}
& 1.00 & 0.90
& \(4.69{\times}10^4 \pm 5.67{\times}10^3\)
& \(4.56{\times}10^6 \pm 5.81{\times}10^5\) \\

\texttt{Slow-only}
& 0.77 & 0.77
& \(5.60{\times}10^6 \pm 8.98{\times}10^6\)
& \(1.63{\times}10^9 \pm 4.22{\times}10^8\) \\
\bottomrule
\end{tabular}
}
\label{tab:main_results}
\end{table}

We evaluate 2FFS on stochastic minimax trees, focusing on whether it can reduce search effort by adaptively combining cheap biased evaluations with expensive accurate samples. We compare against three fixed-confidence baselines: (i) BAI-MCTS sampling (\texttt{BAI-MCTS}) \citep{kaufmann2017monte}, which uses stochastic leaf samples; (ii) minimax fast-oracle (\texttt{Minimax-fast}), which uses only the deterministic fast oracle; and (iii) fixed-depth expansion followed by slow-oracle sampling on the exposed frontier (\texttt{Slow-only}). The second and third baselines are controlled ablations: they test, respectively, the limits of relying only on fast expansion and of committing to a non-adaptive expansion depth before slow sampling. We use the abbreviated names in parentheses in all tables. We note that BAI-MCTS \citep[\S4]{kaufmann2017monte} has shown that it already substantially improves (15$\times$) over earlier tree-based BAI methods, including UGapE-MCTS \citep{gabillon2012best},  LUCB-MCTS \citep{Kalyanakrishnan2012PACSS}, and FindTopWinner \citep{EvenDar2006ActionEA}.
\vspace{-0.5em}
\paragraph{Setup and evaluation.} Our main experiments use balanced \(b\)-ary minimax trees with \( (D,b)\in\{(5,8),(7,6),(10,3)\} \), with \(100\) independently generated trees per setting. These configurations cover both wide shallow trees and deeper narrower trees, and are much larger than the depth-\(3\), \(10\)-ary setting (only $1111$ nodes) in BAI-MCTS \citep{kaufmann2017monte}, with more explicit evidence of efficiency improvement.

For evaluation, we report two complementary metrics. The first is \emph{sampling count}, the total number of nodes sampled and visited. The second is \emph{operation count}, the total number of bookkeeping operation units, where one operation count is one \(\mathcal O(1)\) scalar state access, update, or comparison. This distinction is important because sampling one node  does not necessarily induce the same amount of computation across methods. For example, after a stochastic leaf sample, BAI-MCTS mainly updates the sampled leaf interval and propagates confidence bounds along the leaf-to-root path. In contrast, after a slow oracle sample, 2FFS updates local slow intervals, intersects them with fast intervals, propagates minimax bounds, and maintains scale-dependent comparison certificates. Thus, sampling count measures efficiency of repetitive nodes visit, while operation count measures the overall computation efficiency.

\paragraph{Results.} Table~\ref{tab:main_results} shows that \texttt{2FFS} substantially reduces search effort relative to the \texttt{BAI-MCTS} baseline: across the reported settings, it uses \(160\)--\(1450\times\) fewer sampled or visited nodes and \(1.9\)--\(4.6\times\) fewer bookkeeping operations. The ablations illustrate the two sides of the fast--slow tradeoff. \texttt{Minimax-fast} is computationally cheap because it relies only on fast expansion, but this efficiency comes at the cost of accuracy since the fast oracle is biased. Conversely, \texttt{Slow-only} uses accurate slow-oracle samples and can be accurate when enough samples are collected, but it requires many more samples and often fails to stop within the finite budget. Overall, these results suggest that \texttt{2FFS} gains its efficiency by adaptively choosing when to trust fast expansion and when to spend slow samples on decision-critical nodes.

\begin{figure*}[h]
    \centering
    \includegraphics[width=0.9\linewidth]{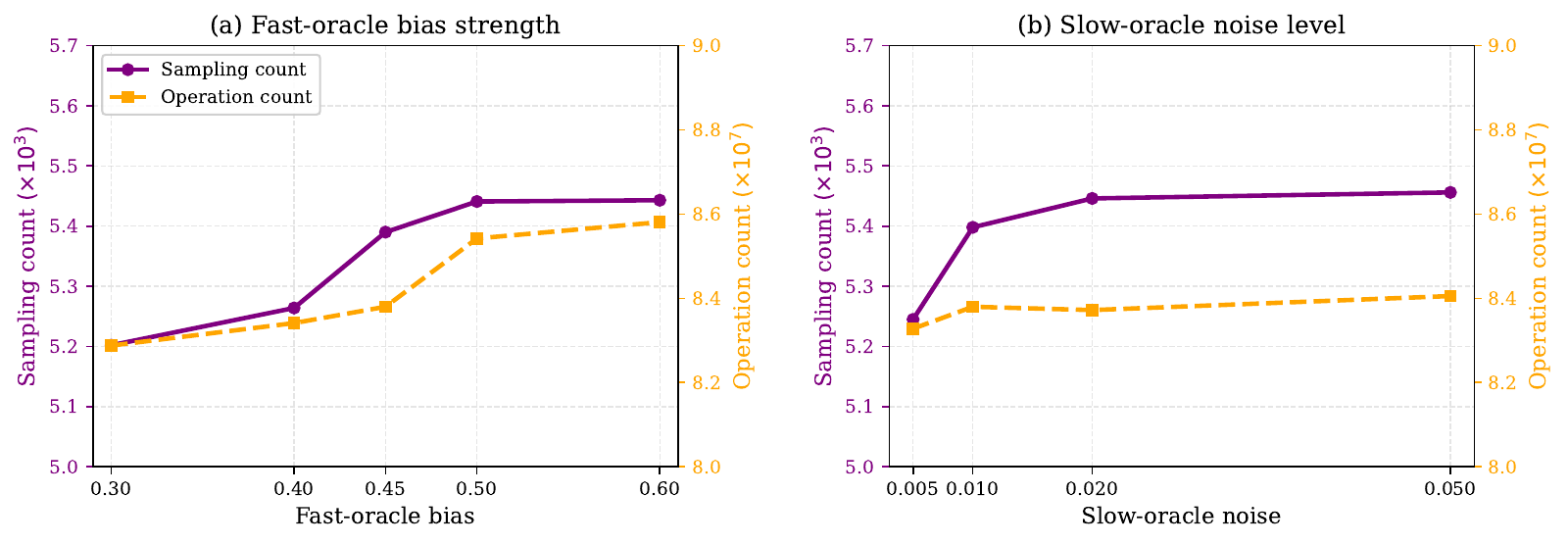}
    \caption{Oracle-bias ablation results for \texttt{2FFS}: (a) fast-oracle bias and (b) slow-oracle noise.}
    \label{f2}
    \vspace{-1em}
\end{figure*}

\paragraph{Ablations.} We further test the sensitivity of \texttt{2FFS} to the fast-oracle bias and slow-oracle noise. Using \(D=5\), 8-ary trees, we sweep the fast-oracle bias with mean \(0.45\) and the slow-oracle noise with mean \(0.01\). Figure~\ref{f2} shows that \texttt{2FFS} is relatively stable across both ranges, suggesting that its performance is not overly sensitive to moderate changes in either oracle parameter.

\section{Conclusion}

In this work, we introduced \texttt{2FFS}, a two-fidelity algorithm for fixed-confidence best-action identification in stochastic minimax trees. Our work bridges multi-fidelity bandits and minimax tree search by providing a principled mechanism for adaptively trading off cheap biased expansion against expensive accurate sampling. Theoretically, we prove PAC correctness, finite stopping, and a general-depth cost bound relative to an ideal recursive oracle complexity. Empirically, \texttt{2FFS} achieves substantial gains in sampling efficiency over BAI-MCTS-style baselines while maintaining reliable root-action identification. These results suggest that two-fidelity tree search may be a useful abstraction beyond the current theoretical framework, with potential applications in multi-agent reinforcement learning, game-tree planning, and LLM policy-rollout systems, where accurate rollouts are expensive but cheap heuristic evaluations are widely available, a further discussion to these is provided in Appendix~\ref{app:related-works}.

\paragraph{Follow-up work.} We note that 2FFS opens up a broad range of potential directions for future work. On the empirical side, we suggest investigating how 2FFS can be integrated with policy-gradient-based MARL, particularly how to design approximate envelopes that ensure the effectiveness of fast oracles. Such extensions could also be explored for training LLMs in multi-agent environments, such as chess playing and adversarial training. On the theoretical side, we first look forward to seeing how 2FFS can be adapted to non-minimax trees, which could facilitate extensions to tree-of-thought rollout structures for expansive LLM reasoning trajectories. Beyond this direction, within minimax trees, we also look forward to further developments in fast-oracle envelope design, as well as weaker assumptions for the cost upper bound, or even tighter upper bounds.

\newpage
\bibliographystyle{unsrtnat}
\bibliography{ref}

\newpage
\newpage
\appendix
\section{Further Related Works and Limitations}\label{app:related-works}

\paragraph{Multi-fidelity bandits and optimization.} Multi-fidelity methods study how to allocate queries across information sources with different costs and accuracies. In flat bandit models, the goal is typically to identify or compete with the best arm while using cheap biased fidelities when they are informative and reserving expensive accurate fidelities for difficult comparisons \citep{kandasamy2016multi,poiani2022multifidelity,wang2023multi,poiani2024optimal}. Related ideas also appear in multi-fidelity Bayesian and black-box optimization, where low-cost simulators, subsampled training runs, or surrogate evaluations are used to accelerate optimization under a total cost budget \citep{Kandasamy2016gaussian,kandasamy2017multi,poloczek2017multi,li2020multi,sen2018multi,pmlr-v108-fiegel20a}. These works provide the broad cost--accuracy motivation for our two-oracle model, but they do not address minimax value propagation through an adversarial tree or fixed-confidence certification of a root action.

\paragraph{Relation to multi-fidelity tree search for noisy black-box optimization.} \citet{sen2019noisy} study noisy black-box optimization with multi-fidelity queries using a tree-search approach. Their tree is a hierarchical partition of the input domain of an unknown function; a node represents a region of the search space, and querying at a fidelity gives a biased and noisy estimate useful for optimistic optimization. Their algorithms, MFHOO and MFPOO, adapt HOO/POO-style optimistic tree search to continuous fidelities and are analyzed through simple regret under a finite cost budget.

This is related in spirit because both settings combine tree search with multi-fidelity information, but it is not the same problem. In our setting, the tree is not an auxiliary partition of a black-box domain: it is the stochastic minimax tree itself, with alternating \textsc{Max}/\textsc{Min} backups and a root-action decision. Our objective is fixed-confidence \((\varepsilon,\delta)\)-PAC best-action identification, not simple-regret minimization after a budget is exhausted. Methodologically, 2FFS maintains valid confidence intervals at minimax nodes, propagates them through adversarial backups, and adaptively chooses between local slow certification and recursive fast expansion. This differs from the optimistic partition-refinement mechanism of MFHOO/MFPOO, where the tree is used to explore a continuous domain and there is no minimax alternation, no root-child certification problem, and no two-route local-versus-recursive certificate of the kind used here.

\paragraph{Best-action identification and tree-based planning.} Pure-exploration bandits distinguish fixed-budget regret-style objectives from fixed-confidence identification objectives. The flat best-arm identification literature characterizes instance-dependent sample complexity for identifying the best arm \citep{kaufmann2016complexity,degenne2019non}, with extensions to structured settings such as unimodal bandits \citep{ghosh2024fixed,poiani2024best} and cost-aware sampling \citep{kanarios2024cost}. In tree-based planning, classical minimax search and \(\alpha\)-\(\beta\) pruning exploit deterministic backups in adversarial trees \citep{Knuth1975alpha,Baudet1978full}, while MCTS uses stochastic sampling and bandit-style selection rules \citep{kocsis2006bandit}. BAI-MCTS \citep{kaufmann2017monte} brings fixed-confidence identification to minimax trees, but assumes a single unbiased sampling oracle. Our contribution combines this fixed-confidence root-action perspective with a calibrated cheap evaluator and an expensive accurate evaluator.

\paragraph{Applications and oracle-based learning.} Multi-fidelity ideas are widely used in hyperparameter optimization \citep{li2018hyperband,jamieson2016non,falkner2018bohb,li2020system}, medical trial design \citep{Robbins1952SomeAO}, and engineering design problems such as hydrofoil, wind-farm, and aerodynamic optimization \citep{BONFIGLIO201863,rethore2014top,Zheng_Hedrick_Mittal_2013}. More broadly, oracle-based learning studies how decision-making systems can exploit structured feedback such as ranking or comparison oracles \citep{tang2023zeroth,chen2025compo}. Our work fits this broader line by studying how cheap heuristic information and expensive reliable feedback should be allocated in a structured decision tree and GDA-style minimax tree \citep{lin2025two}.

\paragraph{Future work.} We acknowledge several limitations, which also open a broad space for follow-up work from both empirical and theoretical perspectives. Empirically, our experiments focus on simulated finite minimax trees, while many modern planning problems involve richer interactive environments (we note that adapting 2FFS to these areas would require the effort of a fully new follow-up work). Extending 2FFS to multi-agent reinforcement learning, policy-gradient-based training, or neural-guided MCTS would be an important step toward broader applications, but is beyond the scope of the present paper. In addition, our current formulation assumes finite trees with fixed depth (which aligns with the setup from previous tree-based BAI works); adapting the method to dynamically generated trees, progressive widening, or very large action spaces is a natural next direction. We also note such minimax planning can be smoothly integrated with the recent group-based language model policy optimization for language model reasoning \citep{Guo-2025-Deepseek,Chen-2025-Stepwise,li2026knapsack,Chen-2025-Exploration} and language model alignment and decision making problems \citep{chen2026reward,park2026post}.

Theoretically, our general-depth upper bound relies on mild local regularity assumptions on the sampling complexity. Understanding whether these conditions can be weakened, verified adaptively, or replaced by fully data-dependent stopping criteria would further sharpen the theory. Another important direction is to close the remaining gap between the current polynomial-depth factor and the ideal recursive oracle complexity, potentially leading to tighter depth dependence or sharper instance-dependent bounds.

\section{Missing Proofs and Further Technical Details}\label{app:technical-details}

This section gives the proofs and algorithmic details omitted from the main text. We begin with a notation table in Appendix~\ref{app:notation-glossary}.

\subsection{Notation Glossary}\label{app:notation-glossary}

\begingroup
\small
\setlength{\LTpre}{0.35em}
\setlength{\LTpost}{0.6em}
\renewcommand{\arraystretch}{1.12}
\begin{longtable}{@{}p{0.27\linewidth}p{0.68\linewidth}@{}}
\caption{Glossary of notation used in the main text and appendix.}
\label{tab:notation-glossary}\\
\toprule
Notation & Meaning \\
\midrule
\endfirsthead
\caption[]{Glossary of notation used in the main text and appendix (continued).}\\
\toprule
Notation & Meaning \\
\midrule
\endhead
\midrule
\multicolumn{2}{r}{\footnotesize Continued on next page}\\
\endfoot
\bottomrule
\endlastfoot
\multicolumn{2}{@{}l}{1) \textbf{Tree, values, and root decision}}\\
\(\mathcal T\), \(r\) & Finite game tree and its root. The root is a \textsc{Max} node. \\
\(\mathrm{Ch}(v)\) & Set of children of node \(v\). \\
\(d(v)\), \(D\), \(h(v)\) & Depth of \(v\), total tree depth, and remaining depth \(h(v)=D-d(v)\). \\
\(\mathcal L\), \(\ell\) & Leaf set and a generic leaf. \\
\(\mu_\ell\), \(V^*(v)\) & Mean payoff at leaf \(\ell\) and true minimax value of node \(v\). \\
\(V^*(r,a)\) & Root-action value, defined as the node value \(V^*(a)\) for \(a\in\mathrm{Ch}(r)\). \\
\(a^*\), \(\Delta_*\) & Unique optimal root action and root gap \(V^*(r,a^*)-\max_{a\ne a^*}V^*(r,a)\). \\
\(\mathcal F\), \(\mathcal F_t\) & A complete frontier cut and the frontier maintained at round \(t\). \\
\(\mathcal T_t\) & Explored subtree at round \(t\). \\
\(p(v)\) & Parent of node \(v\). \\
\(\Delta_v^{\mathrm{eff}}\) & Effective gap of node \(v\), i.e., the largest root-to-\(v\) bottleneck scale that can affect root-action identification. \\
\addlinespace
\multicolumn{2}{@{}l}{2) \textbf{Two-fidelity oracles and sampling}}\\
\(F\), \(S\) & Fast deterministic oracle and slow stochastic oracle. \\
\(V_F(v)\) & Fast-oracle value returned at node \(v\). \\
\(B(h)\) & Known fast-oracle bias envelope at remaining depth \(h\), with \(B(0)=0\). \\
\(\nu_{v,m}\), \(\mu_{v,m}\) & Distribution and mean of fidelity-\(m\) observations at node \(v\). \\
\(\lambda_1,\lambda_2\), \(c\) & Fidelity costs, normalized as \(\lambda_1=1\) and \(\lambda_2=c\ge 1\). \\
\(Y_{v,s}\), \(\sigma\) & The \(s\)-th slow-oracle sample at \(v\), assumed \(\sigma\)-sub-Gaussian around \(V^*(v)\). \\
\(N^F(t)\), \(N_v^S(t)\) & Total number of fast queries and number of slow queries at node \(v\) by round \(t\). \\
\(C_t\) & Total search cost \(N^F(t)+c\sum_v N_v^S(t)\). \\
\(\tau\), \(\hat a_\tau\) & Stopping time and returned root action. \\
\(\varepsilon\), \(\delta\) & PAC accuracy and failure probability. \\
\(\boldsymbol{\delta}\), \(\delta_v\) & Node-wise confidence allocation with \(\sum_{v\ne r}\delta_v\le\delta\). \\
\(\widehat V_{v,n}\) & Empirical mean of the first \(n\) slow samples at \(v\). \\
\(\beta_v(n,\delta_v)\) & Time-uniform confidence radius for node \(v\). \\
\(\mathcal E_\delta\) & Simultaneous-validity event on which all slow confidence intervals are valid. \\
\addlinespace
\multicolumn{2}{@{}l}{3) \textbf{Intervals and local certification}}\\
\(I_v^F\) & Fast interval \([V_F(v)-B(h(v)),\,V_F(v)+B(h(v))]\). \\
\(I_v^S(t)\) & Running intersection of slow confidence intervals at node \(v\). \\
\(I_v^{\mathrm{loc}}(t)\) & Direct local interval \(I_v^F\cap I_v^S(t)\). \\
\(I_v^{\mathrm{ch}}(t)\) & Child-backup interval obtained by minimax propagation from children. \\
\(I_v(t)=[L_v(t),U_v(t)]\) & Effective interval used by the algorithm at node \(v\). \\
\(\mathrm{width}(I)\) & Length of an interval \(I\). In particular, \(\mathrm{width}(I_v(t))=U_v(t)-L_v(t)\). \\
\(m_v(\rho,\delta_v)\) & Minimum slow sample count needed so \(\beta_v(n,\delta_v)\le \rho/4\). \\
\(\Gamma_v(\rho,\delta_v)\) & Local slow-oracle cost to certify node \(v\) at precision \(\rho\); it is zero if the fast interval is already narrow enough. \\
\(J_v^*(\boldsymbol{\delta})\) & Ideal recursive cost to certify subtree \(v\) at precision \(\Delta_v^{\mathrm{eff}}\). \\
\(H(\boldsymbol{\delta})\), \(H^*\) & Root recursive oracle complexity for allocation \(\boldsymbol{\delta}\), and its infimum over feasible allocations. \\
\addlinespace
\multicolumn{2}{@{}l}{4) \textbf{Dyadic scales and algorithmic certificates}}\\
\(\rho_k\) & Dyadic precision grid, initialized at \(\rho_0\) and updated by \(\rho_{k+1}=\rho_k/2\). \\
\(\Gamma_v^{(k)}(\delta_v)\) & Scale-\(k\) local certification cost \(\Gamma_v(\rho_k,\delta_v)\). \\
\(s\in\{L,U\}\), \(\bar s\) & Targeted endpoint side and its opposite side. \\
\(a_t\), \(b_t\) & Current root leader \(\argmax_a L_a(t)\) and challenger \(\argmax_{a\ne a_t}U_a(t)\). \\
\(\lambda_L(v,t)\), \(\lambda_U(v,t)\) & Comparison endpoints \(\max_{u\in\mathrm{Ch}(v)}L_u(t)\) and \(\min_{u\in\mathrm{Ch}(v)}U_u(t)\). \\
\(A_s(v,k,t)\) & Scale-aware active child set for certifying side \(s\) of node \(v\) at scale \(k\). \\
\(b_s(v,k,t)\) & Current blocking child in a comparison case. \\
\(\mathsf{Cert}_s(v,k,t)\) & Observable side certificate at node \(v\), side \(s\), and scale \(k\). \\
\(\mathsf{Done}_s(v,k,t)\) & Latched monotone flag recording that side \(s\) at scale \(k\) was certified earlier. \\
\(\mathsf{Comp}_s(v,k,t)\) & Completion predicate \(\mathsf{Done}_s(v,k,t)\) or \(\mathsf{Cert}_s(v,k,t)\). \\
\(w_x(t)\), \(\underline k(x,t)\) & Current width \(U_x(t)-L_x(t)\) and the dyadic width scale satisfying \(\rho_k/2<w_x(t)\le\rho_k\). \\
\(K_s(x,t)\) & Node-local active scale: the coarsest unresolved scale for side \(s\) at node \(x\). \\
\(K_s^{\le k}(x,t)\) & Capped unresolved scale for side \(s\), restricted to scales no finer than parent scale \(k\). \\
\(\sigma(u,t)\), \(K_{\pm}(u,t)\) & Root-challenger contender side and its corresponding contender scale. \\
\(\Call{Resolve}{v,s,k,B}\) & Recursive resolver call for node \(v\), side \(s\), scale \(k\), and invocation cost cap \(B\). \\
\addlinespace
\multicolumn{2}{@{}l}{5) \textbf{Budgets, counters, and proof accounting}}\\
\(\alpha_h\) & Depth-budget multiplier for remaining depth \(h\), e.g., \((h+1)^2\). \\
\(\mathsf B_{v,k}^{\mathrm{rec}}\) & Recursive race budget \(\alpha_{h(v)}\Gamma_v^{(k)}(\delta_v)\) for node \(v\) and scale \(k\). \\
\(M_{v,k}^{\mathrm{loc}}(t)\) & Cost of slow local samples charged to scale \(k\) at node \(v\) by round \(t\). \\
\(M_{v,k}^{\mathrm{exp}}(t)\) & Cost spent below \(v\) while resolving the scale-\(k\) recursive obligation at \(v\). \\
\(\Lambda_{\mathrm{pre}}\), \(\Lambda_{\mathrm{gap}}\), \(\Lambda_{\mathrm{loc}}\) & Local regularity constants; \(\Lambda_{\mathrm{loc}}=\Lambda_{\mathrm{pre}}\Lambda_{\mathrm{gap}}\). \\
\(P_h\) & Depth-overhead recursion used in the general-depth cost bound. \\
\(\mathcal K_v^{\mathrm{act}}\) & Active scales for \(v\): \(\{k\ge 0:\Delta_v^{\mathrm{eff}}\le 2\rho_k\}\). \\
\(\mathfrak C_v(t)\) & Total cost returned through node \(v\), excluding the fast query exposing \(v\). \\
\(d_v(s,k;t)\) & Minimum inherited cap permitting positive progress on the current unresolved obligation. \\
\(Q_v\) & Recursive comparison budget $P_{h(v)-1}J_v^*(\boldsymbol\delta)$ when the recursive reference route is chosen. \\
\(q_{\min}\) & Minimum positive oracle-action cost, \(\min\{c,1\}\). \\
\end{longtable}
\endgroup

\subsection{Proofs for Preliminaries}\label{app:prelim-proofs}

\subsubsection{Proof of Lemma~\ref{lem:minimax-propagation}}
\begin{proof}
We propagate validity and the width bound upward from the frontier. For any node \(u\) that belongs to \(\mathcal F\) or is an ancestor of a node in \(\mathcal F\), define
\[
m(u)\coloneqq
\max\{d(z)-d(u):z\in\mathcal F,\ u\text{ is an ancestor of }z\text{ or }u=z\}.
\]
The frontier nodes are exactly those with \(m(u)=0\), where validity holds by assumption. We proceed by induction on \(m(u)\).

Let \(m(u)\ge 1\), and assume validity for all smaller values of \(m\). Since \(u\) lies on a path to a frontier node, no strict ancestor of \(u\) belongs to \(\mathcal F\); otherwise, that path would meet \(\mathcal F\) twice. For every child \(x\in\mathrm{Ch}(u)\), each root-to-leaf path through \(x\) must therefore meet \(\mathcal F\) at or below \(x\). Thus \(x\) either belongs to \(\mathcal F\) or has a descendant in \(\mathcal F\). Its propagated interval is valid by the induction hypothesis.

If \(u\) is a \textsc{Max} node, taking maxima of the child bounds gives
\[
\max_x L_x\le \max_x V^*(x)=V^*(u)\le \max_x U_x,
\]
which is exactly \(V^*(u)\in[L_u,U_u]\).

If \(u\) is a \textsc{Min} node, taking minima gives
\[
\min_x L_x\le \min_x V^*(x)=V^*(u)\le \min_x U_x.
\]
This proves the validity statement.

For the half-width bound, use the same induction. Suppose every child interval of \(u\) has width at most \(2w\). If \(u\) is a \textsc{Max} node, choose \(x^*\in\argmax_x U_x\). Then
\[
U_u-L_u = U_{x^*}-\max_x L_x \le U_{x^*}-L_{x^*} \le 2w.
\]
If \(u\) is a \textsc{Min} node, choose \(x^*\in\argmin_x L_x\) to obtain
\[
U_u-L_u = \min_x U_x-L_{x^*} \le U_{x^*}-L_{x^*} \le 2w.
\]
Every propagated ancestor interval therefore has half-width at most \(w\).
\end{proof}

\subsubsection{Proof of Lemma~\ref{lem:interval-validity}}
\begin{proof}
Fix a round \(t\). On \(\mathcal E_\delta\), the slow interval is valid by Eq.~\eqref{eq:slow-interval}, while the fast interval \(I_v^F\) is valid deterministically by Definition~\ref{def:fast}. Their intersection therefore satisfies
\[
V^*(v)\in I_v^{\mathrm{loc}}(t)=I_v^F\cap I_v^S(t),
\]
which holds for every exposed non-root node \(v\).

It remains to check the recursive backups. We use induction upward through the explored tree. If a non-root node \(v\) is unexpanded, then \(I_v^{\mathrm{ch}}(t)=(-\infty,+\infty)\), so validity of \(I_v(t)=I_v^{\mathrm{loc}}(t)\cap I_v^{\mathrm{ch}}(t)\) reduces to local validity.

If \(v\) is an expanded non-root node, all children of \(v\) are explored and have valid effective intervals by the induction hypothesis. For a \textsc{Max} node \(v\), this gives
\[
\max_{u\in\mathrm{Ch}(v)}L_u(t)
\le
\max_{u\in\mathrm{Ch}(v)}V^*(u)
=V^*(v)
\le
\max_{u\in\mathrm{Ch}(v)}U_u(t),
\]
so \(V^*(v)\in I_v^{\mathrm{ch}}(t)\).

The \textsc{Min} case is identical with maxima replaced by minima. Since \(I_v^{\mathrm{loc}}(t)\) is also valid, the intersection \(I_v(t)=I_v^{\mathrm{loc}}(t)\cap I_v^{\mathrm{ch}}(t)\) remains valid.

At the root, the interval consists only of the child backup. Applying the same \textsc{Max} argument to its valid child intervals gives \(V^*(r)\in I_r(t)\), completing the induction.
\end{proof}

\subsection{Full 2FFS Details}\label{app:alg-details}
We now give the full version of Algorithm~\ref{alg:2ffs-t}. The resolver \(\Call{Resolve}{v,s,k,B}\) includes an invocation-level cost cap \(B\), omitted from the main-text pseudocode. This cap prevents a child call from exceeding the remaining budget inherited from its parent. All ties use fixed deterministic node and side orders, independent of the invocation cap and of $\varepsilon$.

To define the algorithm on every sample path, fix a default root action in advance. If an interval intersection becomes empty, terminate the entire run immediately and return this default action, before evaluating any widths or certificates. This exceptional rule never applies on \(\mathcal E_\delta\) by Lemma~\ref{lem:interval-validity}; any error caused by it is therefore included in the confidence-failure event. All subsequent interval and scale operations are performed only while the intervals are nonempty.

The algorithm allows \emph{local reversibility}: once a node \(v\) has been exposed, it may receive direct slow samples at \(v\) even after expansion. Thus the two certification routes at \(v\),
\begin{itemize}
\item the \emph{local route}, which shrinks \(I_v^{\mathrm{loc}}(t)\) through slow samples at \(v\), and
\item the \emph{recursive route}, which shrinks \(I_v^{\mathrm{ch}}(t)\) by expanding below \(v\) and recursively refining its children,
\end{itemize}
remain available throughout the run. Work performed below \(v\) and samples taken at \(v\) both contribute to its effective interval.

The algorithm uses the dyadic target diameters \((\rho_k)_{k\ge 0}\) from \S\ref{sec:prelim}: scale \(k\) means target width \(\rho_k/2\), and \(\Gamma_v^{(k)}(\delta_v)=\Gamma_v(\rho_k,\delta_v)\) is the corresponding local stopping cost.

Each outer-loop iteration computes an \emph{active local scale} for the current leader and challenger, then selects a root-side obligation as specified below. A node \(v\) is resolved at scale \(k\) once the targeted side has an observable certificate at tolerance \(\rho_k/2\).

The choice between local sampling and recursion is governed by a budget. Fix a nondecreasing polynomial depth-budget sequence
\[
\alpha_h\ge 1,\qquad h=0,1,\ldots,D,
\]
for example \(\alpha_h=(h+1)^2\). For an internal node \(v\), the recursive route at scale \(k\) is allowed to spend at most
\[
\mathsf B_{v,k}^{\mathrm{rec}}
\coloneqq
\alpha_{h(v)}\,\Gamma_v^{(k)}(\delta_v)
\]
below \(v\), before the obligation falls back to direct slow sampling at \(v\). The resolver first attempts recursive work and falls back to local sampling when its recursive budget is insufficient. To track spending at each explored non-root node \(v\) and scale \(k\), define
\[
M_{v,k}^{\mathrm{loc}}(t)
\coloneqq
c \times
\bigl(\text{number of slow queries issued at \(v\) while resolving scale \(k\)}\bigr),
\]
and \( M_{v,k}^{\mathrm{exp}}(t)\) as the total cost already spent strictly below \(v\) while resolving the scale-\(k\) recursive obligation at \(v\).

These counters record costs already incurred. The recursive counter enforces the budget, while the current certificates determine which child to refine. Together, these rules limit spending below a node and direct it toward children that still affect the required comparison.

At the root, let
\[
a_t\in \argmax_{a\in \mathrm{Ch}(r)} L_a(t),
\qquad
b_t\in \argmax_{a\in \mathrm{Ch}(r)\setminus\{a_t\}} U_a(t)
\]
be the current leader and challenger. 2FFS stops once some root action beats every competitor up to \(\varepsilon\), i.e.,
\[
\text{stop if there exists } \hat a_t\in \mathrm{Ch}(r)
\text{ such that }
L_{\hat a_t}(t)\ge \max_{a\neq \hat a_t} U_a(t)-\varepsilon.
\]
Otherwise, 2FFS refines one uncertified decision-critical root side.

The challenger may either be excluded or become the leader. The algorithm therefore computes the current lower-side scale for the leader \(a_t\) and the current full-contender scale for the challenger \(b_t\), then selects the obligation with the coarser unresolved diameter.

Inside the selected subtree, the algorithm first attempts recursive work. If the node's own recursive budget prevents progress, the resolver falls back to a direct slow sample, provided its inherited cap permits that sample. If the inherited cap prevents the chosen action, the call returns without spending. In the two selector cases \((\textsc{Min},L)\) and \((\textsc{Max},U)\), the recursive route selects an unresolved endpoint witness. In the two comparison cases \((\textsc{Max},L)\) and \((\textsc{Min},U)\), it selects the child whose endpoint most strongly obstructs the parent-side comparison certificate.

An internal blocker is refined through obligations capped at the parent scale. A call may improve either the parent's comparison side or the opposite endpoint used to exclude the blocker. It cannot use a finer scale than the parent obligation. This restriction is what allows the localization bound below to remain independent of depth.

For an expanded internal node \(v\), define the scale-aware active child set \(A_s(v,k,t)\subseteq \mathrm{Ch}(v)\) as follows. In the two selector cases, we use the endpoint witness sets:
\begin{itemize}
\item if \(s=L\) and \(v\) is a \textsc{Min} node, then
\[
A_L(v,k,t)=\argmin_{u\in \mathrm{Ch}(v)} L_u(t);
\]
\item if \(s=U\) and \(v\) is a \textsc{Max} node, then
\[
A_U(v,k,t)=\argmax_{u\in \mathrm{Ch}(v)} U_u(t);
\]
\end{itemize}
In the two comparison cases we instead compare against the current best comparison endpoint and define a lazy set of critical children:
\begin{itemize}
\item if \(s=L\) and \(v\) is a \textsc{Max} node, let
\[
\lambda_L(v,t)\coloneqq \max_{u\in \mathrm{Ch}(v)} L_u(t),
\]
and set
\[
A_L(v,k,t)=
\bigl\{
u\in \mathrm{Ch}(v):\,
U_u(t)> \lambda_L(v,t)+\rho_k/2
\ \text{and}\ 
\neg\mathsf{Comp}_L(u,k,t)
\bigr\};
\]
\item if \(s=U\) and \(v\) is a \textsc{Min} node, let
\[
\lambda_U(v,t)\coloneqq \min_{u\in \mathrm{Ch}(v)} U_u(t),
\]
and set
\[
A_U(v,k,t)=
\bigl\{
u\in \mathrm{Ch}(v):\,
L_u(t)< \lambda_U(v,t)-\rho_k/2
\ \text{and}\ 
\neg\mathsf{Comp}_U(u,k,t)
\bigr\}.
\]
\end{itemize}
Therefore, \(A_s(v,k,t)\) contains only the children that are still unresolved at the current scale for the comparison required by the parent. In a comparison case, a child is discharged either by endpoint exclusion, namely falling outside the displayed margin, or by same-side completion at scale \(k\). A discharged child is no longer recursively refined for that parent-scale obligation. Equivalently, the two comparison certificates use the following dual discharge rules:
\[
\begin{array}{ll}
\text{\textsc{Max}, \(L\):}
& u\text{ is discharged if }
U_u(t)\le \lambda_L(v,t)+\rho_k/2
\text{ or }\mathsf{Comp}_L(u,k,t),\\[4pt]
\text{\textsc{Min}, \(U\):}
& u\text{ is discharged if }
L_u(t)\ge \lambda_U(v,t)-\rho_k/2
\text{ or }\mathsf{Comp}_U(u,k,t).
\end{array}
\]
So the first line excludes a \textsc{Max}-node child through its upper endpoint or certifies it on the lower side, while the second line is the exact dual for \textsc{Min} nodes.

When \((v,s)\) is a comparison case and \(A_s(v,k,t)\neq \varnothing\), define the current blocking contender by
\[
\begin{cases}
b_L(v,k,t)\in \argmax_{u\in A_L(v,k,t)} U_u(t),
& \text{if } v \text{ is a \textsc{Max} node and } s=L,\\[4pt]
b_U(v,k,t)\in \argmin_{u\in A_U(v,k,t)} L_u(t),
& \text{if } v \text{ is a \textsc{Min} node and } s=U.
\end{cases}
\]
Recursive descent works on this blocker and recomputes it after every unit of work. Other active children are considered when they become the blocker.

2FFS maintains scale-local side-completion flags \(\mathsf{Done}_s(v,k)\in\{0,1\}\). These flags are initialized to \(0\) when the pair \((v,k)\) is first used, and are only changed from \(0\) to \(1\). When used as a predicate, \(\mathsf{Done}_s(v,k,t)\) denotes that the flag is equal to \(1\) at round \(t\). An unstored flag is interpreted as zero; storage is initialized before any assignment to it.

We call \((v,s)\) a comparison case if either \(v\) is a \textsc{Max} node with \(s=L\), or \(v\) is a \textsc{Min} node with \(s=U\). The selector cases are \((\textsc{Min},L)\) and \((\textsc{Max},U)\). In a comparison case, the children in \(A_s(v,k,t)\) are exactly those that are neither endpoint-excluded nor same-side complete at scale \(k\).

Among them, only the current blocker \(b_s(v,k,t)\) is eligible for recursive comparison work. The blocker may be refined either on the parent-targeted side \(s\), which supports promotion or same-side safety, or on the opposite side \(\bar s\), which supports endpoint exclusion, but both calls are capped by the current parent scale \(k\). Note that no child is required to become fully complete merely because it is currently active.

For each fixed scale \(k\) and round \(t\), we define the predicates \(\mathsf{Cert}_s(v,k,t)\) and \(\mathsf{Comp}_s(v,k,t)\) upward from the leaves of the currently explored tree. The observable side certificate \(\mathsf{Cert}_s(v,k,t)\) holds if either
\[
U_v(t)-L_v(t)\le \rho_k/2,
\]
or
\[
\mathrm{width}\bigl(I_v^{\mathrm{loc}}(t)\bigr)\le \rho_k/2,
\]
or \(v\) is an expanded internal node and one of the following recursive alternatives holds:
\[
\begin{cases}
\text{for every }u\in\mathrm{Ch}(v),\
\bigl[
U_u(t)\le \lambda_L(v,t)+\rho_k/2
\ \text{or}\ 
\mathsf{Comp}_L(u,k,t)
\bigr],
& \text{if }v\text{ is \textsc{Max} and }s=L,\\[4pt]
\text{for every }u\in\mathrm{Ch}(v),\
\bigl[
L_u(t)\ge \lambda_U(v,t)-\rho_k/2
\ \text{or}\ 
\mathsf{Comp}_U(u,k,t)
\bigr],
& \text{if }v\text{ is \textsc{Min} and }s=U,\\[4pt]
\mathsf{Comp}_s(u,k,t)\quad\text{for every }u\in A_s(v,k,t),
& \text{if }(v,s)\text{ is a selector case}.
\end{cases}
\]
For leaves, only the first two alternatives can apply. Finally, the following are equivalent
\[
\mathsf{Comp}_s(v,k,t)
\quad\Longleftrightarrow\quad
\mathsf{Done}_s(v,k,t)
\ \text{ or }\
\mathsf{Cert}_s(v,k,t).
\]
This joint definition of \(\mathsf{Cert}\), \(\mathsf{Comp}\), the comparison endpoints \(\lambda_s(v,t)\), and the lazy active sets \(A_s(v,k,t)\) is well-founded, as the only recursive certificate reference is from a node to its children. Every evaluation at a finite scale \(k\), including evaluations inside the scale searches below, first initializes the scale storage at all exposed non-root nodes and then evaluates the predicates in bottom-up order. Whenever a true \(\mathsf{Cert}_s(v,k,t)\) is observed, it is immediately latched by setting \(\mathsf{Done}_s(v,k)\gets1\). Thus a certificate observed during a scale search is retained just as one observed after an oracle update. The routine \Call{EvaluateScale}{k} below implements this convention.

For any explored non-root node \(x\) with width \(w_x(t)=U_x(t)-L_x(t)>0\), let \(\underline k(x,t)\) be the unique integer \(k\ge 0\) such that
\[
\rho_k/2 < w_x(t)\le \rho_k.
\]
This integer exists: balanced depth gives \(\rho_0=2B(D-1)\), and every exposed node has \(w_x(t)\le 2B(h(x))\le\rho_0\), because its effective interval is contained in its fast interval. When \(\rho_0=0\), initialization already determines all root-action values and the stopping test succeeds. The node-local active scale of side \(s\in\{L,U\}\) at \(x\) is
\[
K_s(x,t)
\coloneqq
\min\{k\ge \underline k(x,t):\ \neg\mathsf{Comp}_s(x,k,t)\},
\]
with the convention \(K_s(x,t)=+\infty\) if \(w_x(t)=0\) or if the set is empty. Thus 2FFS never refines a side at a scale where that side is already certified; if the width-derived scale is already complete, it moves directly to the next unresolved finer scale for the same side of the same node.

\paragraph{Finite implementation of the uncapped scale search.} The convention \(K_s=+\infty\) does not require an infinite search. Define \(Z_s(v,t)\) by the same bottom-up formula as \(\mathsf{Cert}_s\), replacing the tolerance \(\rho_k/2\) by zero and every child predicate \(\mathsf{Comp}_{s'}(u,k,t)\) by \(Z_{s'}(u,t)\). This auxiliary evaluation ignores all \(\mathsf{Done}\) flags and neither initializes nor changes storage. If \(Z_s(x,t)\) holds, monotonicity of these formulas in their tolerance shows that \(\mathsf{Comp}_s(x,k,t)\) holds for every finite \(k\), so the search returns \(+\infty\).

Otherwise, the search has a finite upper limit. To see this, collect the strictly positive values among all current effective and local interval widths and all comparison differences \(U_u-\lambda_L(v,t)\) and \(\lambda_U(v,t)-L_u\) appearing in the certificate formulas. This is a finite set, and it is nonempty when \(w_x(t)>0\); let \(b(t)>0\) be its minimum. Let \(k_{\mathrm{stored}}\) be the largest scale currently stored, or \(-1\) if none is stored. Choose a finite \(k_{\mathrm{test}}\ge\underline k(x,t)\) with \(k_{\mathrm{test}}>k_{\mathrm{stored}}\) and \(\rho_{k_{\mathrm{test}}}/2<b(t)\). At this scale all initial flags are zero, and every primitive comparison has the same truth value as at zero tolerance. A bottom-up evaluation therefore gives \(\mathsf{Comp}_s(x,k_{\mathrm{test}},t)=Z_s(x,t)=0\). Evaluating the finitely many scales from \(\underline k(x,t)\) through \(k_{\mathrm{test}}\) returns the first unresolved one. Latching during this finite scan does not alter this argument, since flags at different scales are distinct. Consequently every uncapped scale search terminates after finitely many arithmetic and predicate operations.

The root selection is also well-defined outside $\mathcal E_\delta$ while all intervals remain nonempty. Indeed, bottom-up induction gives
\[
Z_L(v,t)\ \text{and}\ Z_U(v,t)\quad\Longrightarrow\quad U_v(t)=L_v(t).
\]
If either zero-tolerance predicate uses a width alternative, this is immediate. Otherwise, suppose both predicates hold recursively but $v$ has positive width. At a \textsc{Max} node, choose $u\in\argmax_x U_x$. Since $I_v\subseteq I_v^{\mathrm{ch}}$, positive width implies $U_u>\max_x L_x$. The upper selector certificate gives $Z_U(u,t)$, and the lower comparison certificate then requires $Z_L(u,t)$. By induction $U_u=L_u\le\max_x L_x$, a contradiction. The \textsc{Min} case is dual, using a child in $\argmin_x L_x$. An unexpanded node has no recursive certificate, which supplies the induction base. The finite search above therefore shows that a positive-width node cannot have both $K_L=K_U=+\infty$. Whenever the root stopping test fails, $U_{b_t}>L_{a_t}+\varepsilon\ge L_{b_t}$, so the challenger has positive width and a finite contender scale. Thus the selected root obligation is finite on every path on which the algorithm has not already terminated.

For a finite parent scale \(k\), define the capped unresolved scale
\[
K_s^{\le k}(x,t)
\coloneqq
\min\{0\le j\le k:\ \neg\mathsf{Comp}_s(x,j,t)\},
\]
with the convention \(K_s^{\le k}(x,t)=+\infty\) if the set is empty.

A recursive call made while resolving a parent obligation at scale \(k\) may only call a child at a capped scale \(K_{s'}^{\le k}(u,t)\).

Thus, recursive descent can refine a child at a coarser unresolved scale, but never at a scale finer than the current parent obligation.

The root challenger has no parent scale to cap its obligations. Its contender side is defined as follows. If at least one of \(K_L(u,t)\) and \(K_U(u,t)\) is finite, set
\[
\sigma(u,t)\in
\argmax_{s\in\{L,U\}:\,K_s(u,t)<+\infty}\rho_{K_s(u,t)},
\]
breaking ties arbitrarily, and define its contender scale by
\[
K_{\pm}(u,t)\coloneqq K_{\sigma(u,t)}(u,t).
\]
If both \(K_L(u,t)\) and \(K_U(u,t)\) are \(+\infty\), we set \(\sigma(u,t)=U\) arbitrarily and \(K_{\pm}(u,t)=+\infty\). In this case, the contender has no unresolved root-side obligation.

A full-contender call then uses \( \Call{Resolve}{u,\sigma(u,t),K_{\pm}(u,t)}. \) For internal comparison blockers, we instead use capped side obligations. Write \(\bar L=U\) and \(\bar U=L\). During \(\Call{Resolve}{v,s,k}\), an internal blocker \(b_s(v,k,t)\) is first called on the opposite side \(\bar s\) at scale \(K_{\bar s}^{\le k}(b_s(v,k,t),t)\), whenever that capped obligation is finite. This is the endpoint that appears in the exclusion inequality. Only after the blocker has no unresolved opposite-side obligation at scales \(\le k\), while remaining active, is it called on the same side \(s\) at scale \(K_s^{\le k}(b_s(v,k,t),t)\).

The same-side blocker obligation can then discharge the blocker through \(\mathsf{Comp}_s(b_s,k,t)\) or allow it to promote through the comparison-side endpoint. When evaluating a finite scale in \(K_s\) or \(K_s^{\le k}\), the \Call{EvaluateScale}{} convention above initializes and latches the required scale-local flags. The fourth argument of \(\Call{Resolve}{v,s,k,B}\) is an invocation-level cost cap: the call may not perform a positive oracle action whose returned cost would exceed \(B\). We write \(\Call{Resolve}{v,s,k}\) in the analysis when the cap is irrelevant. Each actual recursive call receives the minimum of the parent obligation's remaining recursive budget and its inherited invocation cap.

\begin{remark}[Scale-synchronous progress]\label{rmk:scale-sync-full}
Recursive descent is required to remain synchronized with the current parent obligation and its remaining recursive budget. Concretely, inside a call \(\Call{Resolve}{v,s,k,B}\), a child call \(\Call{Resolve}{u,s',K_{s'}^{\le k}(u,t),B_{\mathrm{rec}}}\) is admissible only if the scale-\(k\) certificate at \(v\) is still unresolved at the current round, the child is selected by the current scale-\(k\) rule at \(v\), and \(B_{\mathrm{rec}}\) is no larger than the remaining budget \(\mathsf B_{v,k}^{\mathrm{rec}}-M_{v,k}^{\mathrm{exp}}\). In particular,
\begin{itemize}
\item in a comparison case, the algorithm may refine only the current blocking contender \(b_s(v,k,t)\in A_s(v,k,t)\); the blocker is refined first on \((b_s(v,k,t),\bar s,K_{\bar s}^{\le k}(b_s(v,k,t),t))\) whenever this obligation is finite, and otherwise on \((b_s(v,k,t),s,K_s^{\le k}(b_s(v,k,t),t))\);
\item in a selector case, the algorithm may refine only a child \(u\in A_s(v,k,t)\) with \(\neg\mathsf{Comp}_s(u,k,t)\), and it calls that child at \(K_s^{\le k}(u,t)\).
\end{itemize}
Every invocation of \(\Call{Resolve}{v,s,k,B}\) performs at most one positive oracle action, either a slow query or an expansion, at the current node or below it. A local query may follow a zero-cost blocked child return. After each positive action, the algorithm updates the intervals, comparison endpoints, blockers, and active child sets before descending again. Thus the coarser-scale tests that could already certify or exclude a child from the current scale-\(k\) obligation are checked before that child is refined further.
\end{remark}

The implementation below is split into reusable routines. Algorithm \ref{alg:2ffs-routines} contains scale and selection utilities, Algorithm \ref{alg:2ffs-actions} contains the atomic oracle actions, Algorithm \ref{alg:2ffs-resolve} is the budgeted recursive resolver, and Algorithm \ref{alg:2ffs-full} is the short top-level driver.

\begin{algorithm}[ht!]
\caption{\textsc{2FFS} scale and selection utilities}
\label{alg:2ffs-routines}
\small
\begin{algorithmic}[1]
\Procedure{EnsureScale}{$v,k$}
    \If{the counters and flags for \((v,k)\) are undefined}
        \State Set \(M_{v,k}^{\mathrm{loc}}\gets 0\) and \(M_{v,k}^{\mathrm{exp}}\gets 0\)
        \State Set \(\mathsf{Done}_L(v,k)\gets 0\) and \(\mathsf{Done}_U(v,k)\gets 0\)
    \EndIf
\EndProcedure
\Procedure{EvaluateScale}{$k$}
    \ForAll{exposed non-root nodes \(y\)}
        \State \Call{EnsureScale}{y,k}
    \EndFor
    \ForAll{exposed non-root nodes \(y\) in bottom-up order}
        \ForAll{\(s'\in\{L,U\}\)}
            \If{\(\mathsf{Cert}_{s'}(y,k,t)\)}
                \State \(\mathsf{Done}_{s'}(y,k)\gets1\)
            \EndIf
        \EndFor
    \EndFor
\EndProcedure
\Procedure{RefreshAndLatch}{$k$}
    \State Recompute affected intervals and propagate them to the root; on any empty intersection, terminate the entire run with the fixed default action
    \State \Call{EvaluateScale}{k}
\EndProcedure
\Procedure{ContenderScale}{$x$}
    \State Compute \(k_L\gets K_L(x,t)\) and \(k_U\gets K_U(x,t)\)
    \If{\(k_L<+\infty\) and \(\bigl(k_U=+\infty\) or \(\rho_{k_L}\ge \rho_{k_U}\bigr)\)}
        \State \textbf{return} \((L,k_L)\)
    \Else
        \State \textbf{return} \((U,k_U)\)
    \EndIf
\EndProcedure
\Procedure{CappedScale}{$x,s,k$}
    \State \textbf{return} \(K_s^{\le k}(x,t)\)
\EndProcedure
\Procedure{RaceBudget}{$v,k$}
    \State \textbf{return} \(\alpha_{h(v)}\Gamma_v^{(k)}(\delta_v)\)
\EndProcedure
\Procedure{RootObligation}{}
    \State Compute \(a\in \argmax_{u\in \mathrm{Ch}(r)} L_u(t)\) and \(b\in \argmax_{u\in \mathrm{Ch}(r)\setminus\{a\}} U_u(t)\)
    \State Compute \(k_a\gets K_L(a,t)\) and \((s_b,k_b)\gets \Call{ContenderScale}{b}\)
    \If{\(k_a<+\infty\) and \(\bigl(k_b=+\infty\) or \(\rho_{k_a}\ge \rho_{k_b}\bigr)\)}
        \State \textbf{return} \((a,L,k_a)\)
    \Else
        \State \textbf{return} \((b,s_b,k_b)\)
    \EndIf
\EndProcedure
\Procedure{ChildObligation}{$v,s,k$}
    \If{\((v,s)\) is a comparison case}
        \State Let \(b\gets b_s(v,k,t)\)
        \State \(k_b^{\bar s}\gets \Call{CappedScale}{b,\bar s,k}\)
        \If{\(k_b^{\bar s}<+\infty\)}
            \State \textbf{return} \((b,\bar s,k_b^{\bar s})\)
        \EndIf
        \State \(k_b^s\gets \Call{CappedScale}{b,s,k}\)
        \State \textbf{return} \((b,s,k_b^s)\)
    \EndIf
    \State Choose \(u\in A_s(v,k,t)\) with \(\neg\mathsf{Comp}_s(u,k,t)\), maximizing \(U_u(t)-L_u(t)\)
    \State \(k_u\gets \Call{CappedScale}{u,s,k}\)
    \State \textbf{return} \((u,s,k_u)\)
\EndProcedure
\end{algorithmic}
\end{algorithm}

\begin{algorithm}[ht!]
\caption{\textsc{2FFS} atomic oracle actions}
\label{alg:2ffs-actions}
\begin{algorithmic}[1]
\Procedure{LocalStep}{$v,k,B$}
    \If{\(B<c\)}
        \State \textbf{return} \((0,\mathsf{blocked})\)
    \EndIf
    \State Query \(S\) once at \(v\) and update \(I_v^S(t)\), \(I_v^{\mathrm{loc}}(t)\), and \(I_v(t)\)
    \State \(M_{v,k}^{\mathrm{loc}}\gets M_{v,k}^{\mathrm{loc}}+c\)
    \State \Call{RefreshAndLatch}{k}
    \State \textbf{return} \((c,\mathsf{progress})\)
\EndProcedure
\Procedure{ExpandNode}{$v,k$}
    \State Add \(\mathrm{Ch}(v)\) to \(\mathcal{T}_t\) and query \(F\) on each child
    \ForAll{\(u\in \mathrm{Ch}(v)\)}
        \State \Call{EnsureScale}{u,k}
    \EndFor
    \State \(M_{v,k}^{\mathrm{exp}}\gets M_{v,k}^{\mathrm{exp}}+|\mathrm{Ch}(v)|\)
    \State \Call{RefreshAndLatch}{k}
    \State \textbf{return} \((|\mathrm{Ch}(v)|,\mathsf{progress})\)
\EndProcedure
\end{algorithmic}
\end{algorithm}

\begin{algorithm}[ht!]
\caption{\textsc{2FFS} budgeted recursive resolver}
\label{alg:2ffs-resolve}
\begin{algorithmic}[1]
\Procedure{Resolve}{$v,s,k,B$}
    \State \Call{EnsureScale}{v,k}
    \If{\(\mathsf{Comp}_s(v,k,t)\)}
        \State Set \(\mathsf{Done}_s(v,k)\gets 1\)
        \State \textbf{return} \((0,\mathsf{cert})\)
    \EndIf
    \If{\(h(v)=0\)}
        \State \textbf{return} \Call{LocalStep}{v,k,B}
    \EndIf
    \State \(R\gets \Call{RaceBudget}{v,k}-M_{v,k}^{\mathrm{exp}}\)
    \If{\(R\le 0\)}
        \State \textbf{return} \Call{LocalStep}{v,k,B}
    \EndIf
    \State \(B_{\mathrm{rec}}\gets \min\{B,R\}\)
    \If{\(v\) is unexpanded}
        \If{\(|\mathrm{Ch}(v)|>R\)}
            \State \textbf{return} \Call{LocalStep}{v,k,B}
        \ElsIf{\(|\mathrm{Ch}(v)|>B\)}
            \State \textbf{return} \((0,\mathsf{blocked})\)
        \EndIf
        \State \textbf{return} \Call{ExpandNode}{v,k}
    \EndIf
    \State \((u,s_u,k_u)\gets \Call{ChildObligation}{v,s,k}\)
    \State \((q,z)\gets \Call{Resolve}{u,s_u,k_u,B_{\mathrm{rec}}}\)
    \If{\(z=\mathsf{blocked}\)}
        \If{\(B<R\)}
            \State \textbf{return} \((0,\mathsf{blocked})\)
        \EndIf
        \State \textbf{return} \Call{LocalStep}{v,k,B}
    \EndIf
    \State \(M_{v,k}^{\mathrm{exp}}\gets M_{v,k}^{\mathrm{exp}}+q\)
    \State \Call{RefreshAndLatch}{k}
    \State \textbf{return} \((q,z)\)
\EndProcedure
\end{algorithmic}
\end{algorithm}

\begin{algorithm}[ht!]
\caption{\textsc{2FFS} top-level driver}
\label{alg:2ffs-full}
\begin{algorithmic}[1]
\State \textbf{Input:} node-wise confidence schedule \((\delta_v)\), accuracy \(\varepsilon\), depth budgets \((\alpha_h)_{h=0}^D\)
\State \(\mathcal{T}_0\gets\{r\}\cup \mathrm{Ch}(r)\)
\ForAll{\(a\in \mathrm{Ch}(r)\)}
    \State Query \(F\) at \(a\) and initialize \(I_a^F\), \(I_a^{\mathrm{loc}}(0)\), and \(I_a(0)\)
\EndFor
\State Propagate intervals to the root through Eq.~\eqref{eq:minmax-propegate}
\State Set $\rho_0\gets\max_{a\in\mathrm{Ch}(r)}(U_a(0)-L_a(0))$ and $\rho_k\gets2^{-k}\rho_0$ for $k\ge0$ (evaluated as needed)
\While{there is no \(\hat a_t\in \mathrm{Ch}(r)\) with \(L_{\hat a_t}(t)\ge \max_{a\neq \hat a_t} U_a(t)-\varepsilon\)}
    \State \((x,s,k)\gets \Call{RootObligation}{}\)
    \State \Call{Resolve}{$x,s,k,+\infty$}
\EndWhile
\State \textbf{return} any root child \(\hat a_\tau\) with \(L_{\hat a_\tau}(\tau)\ge \max_{a\neq \hat a_\tau} U_a(\tau)-\varepsilon\)
\end{algorithmic}
\end{algorithm}

\paragraph{Lazy critical children.} In the selector cases, the live sets reduce to the current endpoint witness set. In the comparison cases, a child is discharged once it is endpoint-excluded by the scale-\(k\) comparison margin or same-side complete at scale \(k\). Among the still-live children, only the endpoint-most blocker is eligible for recursive comparison work. This blocker is refined through capped side obligations only, so it may either become safe on the comparison side or be excluded through the opposite endpoint without being pushed to a finer scale than the parent.

\subsection{Proof of Theorem~\ref{thm:pac-correct}}
\label{app:pac-correct}
\begin{proof}
By Eq.~\eqref{eq:E-delta} and the union bound in \S\ref{sec:prelim},
\[
\mathbb P(\mathcal E_\delta)\ge 1-\delta.
\]
We prove that no \(\varepsilon\)-incorrect finite recommendation can occur on \(\mathcal E_\delta\). Fix an outcome in \(\mathcal E_\delta\) and suppose the algorithm stops at a finite time \(\tau\). Let \(\hat a_\tau\) be any returned root child. Interval validity rules out the exceptional empty-intersection exit on this outcome. Hence the ordinary stopping rule gives
\[
L_{\hat a_\tau}(\tau)
\ge
\max_{a\neq \hat a_\tau}U_a(\tau)-\varepsilon.
\]
Interval validity holds simultaneously at every round, so Lemma~\ref{lem:interval-validity} also applies at the finite stopping time \(\tau\). For every \(a\in\mathrm{Ch}(r)\),
\[
L_a(\tau)\le V^*(r,a)\le U_a(\tau).
\]
Combining validity with the stopping inequality gives
\[
V^*(r,\hat a_\tau)
\ge
L_{\hat a_\tau}(\tau)
\ge
\max_{a\neq \hat a_\tau}U_a(\tau)-\varepsilon
\ge
\max_{a\neq \hat a_\tau}V^*(r,a)-\varepsilon.
\]
If \(\hat a_\tau=a^*\), then the returned action is optimal. Otherwise \(a^*\neq \hat a_\tau\), and the last display with \(a=a^*\) gives
\[
V^*(r,a^*)-V^*(r,\hat a_\tau)\le \varepsilon.
\]
Thus, every finite recommendation on \(\mathcal E_\delta\) is \(\varepsilon\)-optimal, and the error event in the theorem is contained in \(\mathcal E_\delta^c\), whose probability is at most \(\delta\).
\end{proof}

\subsection{Derived Lemmas for the General-Depth Bound}\label{app:general-depth-proofs}
We first establish soundness of the side certificates and then show that recursive calls remain at scales compatible with the effective gaps. These facts will support the cost bound in Theorem~\ref{thm:budgeted-general-depth}. Throughout, we work on the simultaneous-validity event \(\mathcal E_\delta\), set \(\varepsilon=0\), and use Assumption~\ref{ass:local-scale-regularity} only in the aggregate summation and charging steps. All dyadic statements below concern $\rho_0>0$. If $\rho_0=0$, then $B(D-1)=0$ and monotonicity of $B$ gives $B(h(v))=0$ for every non-root node. Initialization already identifies the optimal root action, and every $\Gamma_v(\Delta_v^{\mathrm{eff}},\delta_v)$ and $J_v^*$ is zero. Hence $C_\tau=H(\boldsymbol\delta)=|\mathrm{Ch}(r)|$, without evaluating any dyadic cost at zero precision or requiring dyadic regularity.

The certificates control one endpoint at a time. We use the following notion of safety to express that control.
\begin{definition}[Side-specific scale safety]\label{def:side-safe}
Fix a scale index \(k\ge 0\), a round \(t\), an explored non-root node \(v\), and a side \(s\in\{L,U\}\). We say that \(v\) is \((s,k)\)-safe at round \(t\) if
\[
\begin{cases}
V^*(v)-L_v(t)\le \rho_k/2, & \text{if } s=L,\\[4pt]
U_v(t)-V^*(v)\le \rho_k/2, & \text{if } s=U.
\end{cases}
\]
\end{definition}

On \(\mathcal E_\delta\), both one-sided errors in Definition~\ref{def:side-safe} are nonnegative by Lemma~\ref{lem:interval-validity}. Moreover, if
\[
U_v(t)-L_v(t)\le \rho_k/2,
\]
then \(v\) is both \((L,k)\)-safe and \((U,k)\)-safe.

\begin{lemma}[Nested interval updates]\label{lem:interval-nesting}
For every node \(v\) whose effective interval is defined at round \(t\) including the root, the effective intervals are nested over time: if \(t'\ge t\), then
\[
I_v(t')\subseteq I_v(t).
\]
Equivalently, \(L_v(t')\ge L_v(t)\) and \(U_v(t')\le U_v(t)\). The same nesting holds for \(I_v^{\mathrm{loc}}(t)\) whenever \(v\neq r\) is exposed by round \(t\).
\end{lemma}

\begin{proof}
We check the local intervals first, then propagate nesting up the explored tree. At an exposed non-root node, the fast interval is fixed, while the running intersection in Eq.~\eqref{eq:slow-interval} can only shrink as slow samples arrive. Therefore \(I_v^{\mathrm{loc}}(t)=I_v^F\cap I_v^S(t)\) is nested.

For the child backups, both minimax operators preserve the required endpoint monotonicity. At a \textsc{Max} node, the lower endpoint \(\max_u L_u(t)\) is nondecreasing and the upper endpoint \(\max_u U_u(t)\) is nonincreasing whenever every child interval is nested.

At a \textsc{Min} node, the corresponding endpoints are \(\min_u L_u(t)\) and \(\min_u U_u(t)\), with the same monotonicity. Induction from the leaves of the explored tree therefore shows that \(I_v^{\mathrm{ch}}(t)\) is nested after expansion. Before expansion it is \((-\infty,+\infty)\), so the first expansion preserves nesting as well. Intersecting the local and child-backup intervals proves the claim for \(I_v(t)\) at every non-root node. At the root, the effective interval is the child backup itself.
\end{proof}

Interval nesting preserves an endpoint-accuracy guarantee once established. We now verify that the side certificates provide such a guarantee.

\begin{lemma}[Comparison-certificate soundness]\label{lem:cert-soundness}
Fix a round \(t\), a scale index \(k\ge 0\), an explored non-root node \(v\), and a side \(s\in\{L,U\}\). On \(\mathcal E_\delta\), if \(\mathsf{Comp}_s(v,k,t)\) holds, then \(v\) is \((s,k)\)-safe at round \(t\).
\end{lemma}

\begin{proof}
We induct on the remaining depth \(h(v)\), with the induction claim uniform over all rounds, scales, and sides: whenever \(\mathsf{Comp}_s(v,k,t)\) holds, \(v\) is \((s,k)\)-safe at round \(t\). Fix \(v\) and assume the claim for every child of \(v\). We first verify the current certificate \(\mathsf{Cert}_s(v,k,t)\), then use interval nesting to handle certificates latched at an earlier round.

If
\[
U_v(t)-L_v(t)\le \rho_k/2,
\]
then Lemma~\ref{lem:interval-validity} gives \(V^*(v)\in [L_v(t),U_v(t)]\), so both one-sided errors are at most \(\rho_k/2\).

If
\[
\mathrm{width}\bigl(I_v^{\mathrm{loc}}(t)\bigr)\le \rho_k/2,
\]
then \(I_v(t)\subseteq I_v^{\mathrm{loc}}(t)\), while interval validity places \(V^*(v)\) in \(I_v(t)\). Thus this alternative also bounds both one-sided errors by \(\rho_k/2\).

It remains to consider \(\mathsf{Cert}_s(v,k,t)\) obtained through the recursive alternative. The selector cases use an endpoint witness; the comparison cases require a bound for every child.

\emph{\textsc{Min} node, \(s=L\).} The selector set is
\[
A_L(v,k,t)=\arg\min_{u\in\mathrm{Ch}(v)}L_u(t),
\]
which is nonempty. The recursive certificate requires every \(u\in A_L(v,k,t)\) to satisfy \(\mathsf{Comp}_L(u,k,t)\); the induction hypothesis therefore makes each such child \((L,k)\)-safe. Choose any \(x\in A_L(v,k,t)\). Its endpoint is the selector minimum, and the parent value is no larger than its value:
\[
L_x(t)=\min_{u\in\mathrm{Ch}(v)}L_u(t)
\qquad\text{and}\qquad
V^*(v)=\min_{u\in\mathrm{Ch}(v)}V^*(u)\le V^*(x),
\]
and containment \(I_v(t)\subseteq I_v^{\mathrm{ch}}(t)\) gives
\[
L_v(t)\ge \min_{u\in\mathrm{Ch}(v)}L_u(t)=L_x(t),
\]
so the selected child's safety yields
\[
V^*(v)-L_v(t)
\le
V^*(x)-L_x(t)
\le
\rho_k/2.
\]

\emph{\textsc{Max} node, \(s=U\).} The selector set is
\[
A_U(v,k,t)=\arg\max_{u\in\mathrm{Ch}(v)}U_u(t),
\]
which is again nonempty. Every \(u\in A_U(v,k,t)\) satisfies \(\mathsf{Comp}_U(u,k,t)\) and is therefore \((U,k)\)-safe by the induction hypothesis. Choose any \(x\in A_U(v,k,t)\). The selector maximum and the minimax definition give
\[
U_x(t)=\max_{u\in\mathrm{Ch}(v)}U_u(t)
\qquad\text{and}\qquad
V^*(v)=\max_{u\in\mathrm{Ch}(v)}V^*(u)\ge V^*(x),
\]
and containment \(I_v(t)\subseteq I_v^{\mathrm{ch}}(t)\) gives
\[
U_v(t)\le \max_{u\in\mathrm{Ch}(v)}U_u(t)=U_x(t),
\]
so the selected child's safety yields
\[
U_v(t)-V^*(v)
\le
U_x(t)-V^*(x)
\le
\rho_k/2.
\]

\emph{\textsc{Max} node, \(s=L\).} Let \(\lambda\coloneqq \lambda_L(v,t)=\max_{w\in\mathrm{Ch}(v)}L_w(t)\). Fix any child \(u\in\mathrm{Ch}(v)\). If it is endpoint-excluded, then
\[
V^*(u)\le U_u(t)\le \lambda+\rho_k/2.
\]
If the child is not endpoint-excluded, the certificate requires \(\mathsf{Comp}_L(u,k,t)\). The induction hypothesis then gives
\[
V^*(u)\le L_u(t)+\rho_k/2\le \lambda+\rho_k/2.
\]
Taking the minimax backup over these child bounds gives
\[
V^*(v)=\max_{u\in\mathrm{Ch}(v)}V^*(u)\le \lambda+\rho_k/2.
\]
Since \(I_v(t)\subseteq I_v^{\mathrm{ch}}(t)\) and \(\lambda=\max_u L_u(t)\),
\[
L_v(t)\ge \max_{u\in\mathrm{Ch}(v)}L_u(t)=\lambda.
\]
Combining the last two bounds yields
\[
V^*(v)-L_v(t)\le \rho_k/2.
\]

\emph{\textsc{Min} node, \(s=U\).} We now bound every child from below. Let \(\lambda\coloneqq \lambda_U(v,t)=\min_{w\in\mathrm{Ch}(v)}U_w(t)\). Fix any child \(u\in\mathrm{Ch}(v)\). If it is endpoint-excluded, then
\[
V^*(u)\ge L_u(t)\ge \lambda-\rho_k/2.
\]
If the child is not endpoint-excluded, the certificate requires \(\mathsf{Comp}_U(u,k,t)\), so the induction hypothesis gives
\[
V^*(u)\ge U_u(t)-\rho_k/2\ge \lambda-\rho_k/2.
\]
Taking the minimax backup over these child bounds gives
\[
V^*(v)=\min_{u\in\mathrm{Ch}(v)}V^*(u)\ge \lambda-\rho_k/2.
\]
Since \(I_v(t)\subseteq I_v^{\mathrm{ch}}(t)\) and \(\lambda=\min_u U_u(t)\),
\[
U_v(t)\le \min_{u\in\mathrm{Ch}(v)}U_u(t)=\lambda.
\]
Combining the last two bounds yields
\[
U_v(t)-V^*(v)\le \rho_k/2.
\]

This establishes soundness of \(\mathsf{Cert}_s(v,k,t)\), including the leaf case, where only the width alternatives apply.

We next check that latching preserves this guarantee. Suppose \(\mathsf{Done}_s(v,k,t)=1\), and let \(t_0\le t\) be the first round at which the flag changed from \(0\) to \(1\). If it was set by a bottom-up certificate evaluation, during either a scale search or a post-update sweep, then \(\mathsf{Cert}_s(v,k,t_0)\) held by the update rule. Otherwise it was set by the first branch of \(\Call{Resolve}{v,s,k}\). In this case, the flag was still zero before the assignment, so the branch condition \(\mathsf{Comp}_s(v,k,t_0)\) could only hold through \(\mathsf{Cert}_s(v,k,t_0)\). Hence either way, \(\mathsf{Cert}_s(v,k,t_0)\) holds.

The certificate argument above applies at the earlier round \(t_0\), so \(v\) was \((s,k)\)-safe at round \(t_0\). If \(s=L\), Lemma~\ref{lem:interval-nesting} gives \(L_v(t)\ge L_v(t_0)\), and therefore
\[
V^*(v)-L_v(t)\le V^*(v)-L_v(t_0)\le \rho_k/2.
\]
If \(s=U\), then \(U_v(t)\le U_v(t_0)\), and thus
\[
U_v(t)-V^*(v)\le U_v(t_0)-V^*(v)\le \rho_k/2.
\]
Thus a latched side remains safe. Since \(\mathsf{Comp}_s(v,k,t)\) is the disjunction of \(\mathsf{Cert}_s(v,k,t)\) and \(\mathsf{Done}_s(v,k,t)\), the lemma follows.
\end{proof}

We next verify that an unresolved parent always leaves a child obligation at a scale no finer than its own. This ensures that the recursive selection rule is defined.

\begin{lemma}[Capped recursive progress]\label{lem:capped-recursive-progress}
Fix a round \(t\), a finite scale \(k\ge 0\), an expanded internal node \(v\), and a side \(s\in\{L,U\}\). If \(\neg\mathsf{Comp}_s(v,k,t)\), then the current scale-\(k\) recursive rule has at least one finite capped child obligation whenever recursive work is attempted, i.e.,
\begin{itemize}
\item if \((v,s)\) is a selector case, then there exists \(u\in A_s(v,k,t)\) such that
\[
\neg\mathsf{Comp}_s(u,k,t)
\qquad\text{and}\qquad
K_s^{\le k}(u,t)<+\infty;
\]
\item if \((v,s)\) is a comparison case, then \(A_s(v,k,t)\neq\varnothing\); for \(b=b_s(v,k,t)\),
\[
K_{\bar s}^{\le k}(b,t)<+\infty
\qquad\text{or}\qquad
K_s^{\le k}(b,t)<+\infty.
\]
\end{itemize}
Thus, when the budget permits, both the comparison and selector recursive branches have an admissible capped child whenever the parent side predicate \(\mathsf{Comp}_s(v,k,t)\) is false.
\end{lemma}

\begin{proof}
We show that failure of the parent certificate leaves an unresolved child within the permitted range of scales. Since \(\mathsf{Comp}_s\) includes \(\mathsf{Cert}_s\), \(\neg\mathsf{Comp}_s(v,k,t)\) implies \(\neg\mathsf{Cert}_s(v,k,t)\).

Suppose first that \((v,s)\) is a selector case. If every child \(u\in A_s(v,k,t)\) satisfied \(\mathsf{Comp}_s(u,k,t)\), then the recursive selector alternative in the definition of \(\mathsf{Cert}_s(v,k,t)\) would hold, contradicting \(\neg\mathsf{Cert}_s(v,k,t)\). Thus there is a child \(u\in A_s(v,k,t)\) with \(\neg\mathsf{Comp}_s(u,k,t)\). Since the capped set defining \(K_s^{\le k}(u,t)\) contains the index \(k\), this implies
\[
K_s^{\le k}(u,t)<+\infty.
\]
Now suppose \((v,s)\) is a comparison case. We first check that the active set is nonempty. If \(A_s(v,k,t)=\varnothing\), then every child is discharged at scale \(k\): in the \textsc{Max}, \(L\), case every child \(u\) satisfies
\[
U_u(t)\le \lambda_L(v,t)+\rho_k/2
\qquad\text{or}\qquad
\mathsf{Comp}_L(u,k,t),
\]
and in the \textsc{Min}, \(U\), case every child \(u\) satisfies
\[
L_u(t)\ge \lambda_U(v,t)-\rho_k/2
\qquad\text{or}\qquad
\mathsf{Comp}_U(u,k,t).
\]
These are the recursive discharge conditions for \(\mathsf{Cert}_s(v,k,t)\), contradicting \(\neg\mathsf{Cert}_s(v,k,t)\). Hence \(A_s(v,k,t)\neq\varnothing\).

Let \(b=b_s(v,k,t)\) be the selected blocker. Its membership \(b\in A_s(v,k,t)\) gives
\[
\neg\mathsf{Comp}_s(b,k,t).
\]
Thus the set defining \(K_s^{\le k}(b,t)\) contains \(k\), and consequently
\[
K_s^{\le k}(b,t)<+\infty.
\]
If \(K_{\bar s}^{\le k}(b,t)<+\infty\), the algorithm refines that exclusion side. If \(K_{\bar s}^{\le k}(b,t)=+\infty\), the finite same-side obligation is available instead. In either case, the comparison rule has an admissible child while \(\mathsf{Comp}_s(v,k,t)\) is false.
\end{proof}

The latching rule also prevents an obligation from incurring work after it has been certified.

\begin{lemma}[No same-scale rework]\label{lem:no-same-scale-rework}
Fix a finite scale \(k\), an explored non-root node \(v\), and a side \(s\in\{L,U\}\). If \(\Call{Resolve}{v,s,k}\) begins at round \(t\) with \(\mathsf{Comp}_s(v,k,t)\), then it sets \(\mathsf{Done}_s(v,k)\gets 1\) and returns \((0,\mathsf{cert})\) before issuing any oracle query, expanding any node, making any child call, or increasing any work counter. Consequently, once \(\mathsf{Done}_s(v,k)=1\), the node remains \((s,k)\)-safe on \(\mathcal E_\delta\), and every later invocation of \(\Call{Resolve}{v,s,k}\) returns zero before doing positive work.
\end{lemma}

\begin{proof}
After initializing any unused scale-local storage, \(\Call{Resolve}{v,s,k}\) first tests \(\mathsf{Comp}_s(v,k,t)\). If this predicate holds, the procedure latches \(\mathsf{Done}_s(v,k)\) and immediately returns \((0,\mathsf{cert})\). The procedure reaches every oracle query, expansion, child call, and counter increment only after this guard fails.

Each flag \(\mathsf{Done}_s(v,k)\) is initialized to \(0\) and can change only from \(0\) to \(1\). Once latched, \(\mathsf{Comp}_s(v,k,t')\) holds at every later round \(t'\) by definition. Lemma~\ref{lem:cert-soundness} then gives \((s,k)\)-safety on \(\mathcal E_\delta\) at every such \(t'\). The initial guard at round \(t'\) therefore returns before any positive work, regardless of subsequent changes to the comparison endpoint, blocker, or active set.
\end{proof}

We next convert the first unresolved scale into an endpoint-error bound. The capped form applies to recursive child selection.

\begin{lemma}[Endpoint control at capped scales]\label{lem:capped-endpoint-control}
Assume \(\mathcal E_\delta\). Fix a finite parent cap \(\kappa\), a round \(t\), an explored non-root node \(x\), and a side \(s\in\{L,U\}\). If
\[
j=K_s^{\le \kappa}(x,t)<+\infty,
\]
then
\[
\begin{cases}
V^*(x)-L_x(t)\le \rho_j, & s=L,\\[3pt]
U_x(t)-V^*(x)\le \rho_j, & s=U.
\end{cases}
\]
If \(K_s^{\le \kappa}(x,t)=+\infty\), then \(\mathsf{Comp}_s(x,\kappa,t)\) holds, and therefore
\[
\begin{cases}
V^*(x)-L_x(t)\le \rho_\kappa/2, & s=L,\\[3pt]
U_x(t)-V^*(x)\le \rho_\kappa/2, & s=U.
\end{cases}
\]
\end{lemma}

\begin{proof}
If \(K_s^{\le \kappa}(x,t)=+\infty\), the capped set contains no unresolved scale: every index \(0\le i\le \kappa\) is complete on side \(s\), in particular \(i=\kappa\). Lemma~\ref{lem:cert-soundness}, applied at the cap, proves the bound.

Now let \(j=K_s^{\le \kappa}(x,t)<+\infty\). If \(U_x(t)=L_x(t)\), interval validity gives zero error. Otherwise, let \(\ell=\underline k(x,t)\). We first show \(j\ge \ell\), which places the selected scale within the range controlled by the current width. If \(j<\ell\), then the dyadic definition of \(\ell\) gives
\[
U_x(t)-L_x(t)\le \rho_{\ell}\le \rho_j/2,
\]
using \(\rho_{i+1}=\rho_i/2\). The width alternative in \(\mathsf{Cert}_s(x,j,t)\) would hold, contradicting the defining property \(\neg\mathsf{Comp}_s(x,j,t)\) of \(j\).

If \(j=\ell\), the width itself controls either endpoint error, since
\[
U_x(t)-L_x(t)\le \rho_{\ell}=\rho_j.
\]
If \(j>\ell\), then \(j-1\ge 0\). Also \(j\le \kappa\), because \(j\) belongs to the capped set, so \(j-1\le \kappa\). The minimality of \(j\) therefore gives \(\mathsf{Comp}_s(x,j-1,t)\). Lemma~\ref{lem:cert-soundness} bounds the endpoint error by
\[
\rho_{j-1}/2=\rho_j.
\]
These two cases prove the finite-\(j\) claim for either side.
\end{proof}

We can now bound the value gap across a selected edge. The argument uses only the current selection rule, so it also covers a child call blocked by its inherited budget.

\begin{lemma}[One-step localization under capped descent]\label{lem:capped-one-step-localization}
Assume \(\mathcal E_\delta\). Fix a round \(t\), an expanded internal node \(v\), a side \(s\), and a finite scale \(k\) such that \(\neg\mathsf{Comp}_s(v,k,t)\). Let \((u,s_u,k_u)\) be the output of \(\Call{ChildObligation}{v,s,k}\) evaluated at this state. Then \(k_u<+\infty\), \(k_u\le k\), and
\[
|V^*(u)-V^*(v)|\le \rho_{k_u}.
\]
\end{lemma}

\begin{proof}
We relate the child's value gap to the endpoint that caused its selection. Lemma~\ref{lem:capped-recursive-progress} first ensures that the selected obligation is finite. The rule uses the capped scale \(k_u=K_{s_u}^{\le k}(u,t)\), whether or not the available budget permits an actual call. Thus \(k_u\le k\); the dyadic relation \(\rho_{i+1}=\rho_i/2\) then gives \(\rho_k\le \rho_{k_u}\). Write \(\bar\rho=\rho_k\) and \(\rho=\rho_{k_u}\).

\emph{Selector cases.} If \(v\) is a \textsc{Min} node and \(s=L\), then \(u\in\arg\min_x L_x(t)\). Let \(w\) be a child attaining \(V^*(v)=\min_x V^*(x)\). The selected lower endpoint is accurate by Lemma~\ref{lem:capped-endpoint-control}:
\[
V^*(u)-L_u(t)\le \rho.
\]
The selector ordering gives \(L_u(t)\le L_w(t)\le V^*(w)=V^*(v)\). Together with \(V^*(u)\ge V^*(v)\), this places the value gap inside the endpoint error:
\[
|V^*(u)-V^*(v)|=V^*(u)-V^*(v)\le V^*(u)-L_u(t)\le \rho.
\]
For the \textsc{Max}, \(s=U\), selector case, the same argument uses the upper endpoints. Now \(u\in\arg\max_x U_x(t)\), so a child \(w\) attaining \(V^*(v)=\max_x V^*(x)\) satisfies
\[
U_u(t)\ge U_w(t)\ge V^*(w)=V^*(v).
\]
Lemma~\ref{lem:capped-endpoint-control} gives \(U_u(t)-V^*(u)\le\rho\), and since \(V^*(u)\le V^*(v)\),
\[
|V^*(u)-V^*(v)|=V^*(v)-V^*(u)\le U_u(t)-V^*(u)\le \rho.
\]

\emph{Comparison at a \textsc{Max} node.} Suppose \(v\) is \textsc{Max} and \(s=L\). Since \(\mathsf{Comp}_L(v,k,t)\) is false, Lemma~\ref{lem:capped-recursive-progress} gives \(A_L(v,k,t)\neq\varnothing\). Let \(\lambda=\lambda_L(v,t)\) and \(b=b_L(v,k,t)\), so the selected child is \(u=b\). We first establish \(V^*(v)\le U_b(t)\) by bounding every child value. For a child \(x\) with \(x\in A_L(v,k,t)\), the definition of \(b\) gives
\[
V^*(x)\le U_x(t)\le U_b(t).
\]
If \(x\notin A_L(v,k,t)\), then either
\[
U_x(t)\le \lambda+\bar\rho/2,
\]
or \(\mathsf{Comp}_L(x,k,t)\). In the latter case, Lemma~\ref{lem:cert-soundness} gives
\[
V^*(x)\le L_x(t)+\bar\rho/2\le \lambda+\bar\rho/2.
\]
Membership \(b\in A_L(v,k,t)\) gives \(U_b(t)>\lambda+\bar\rho/2\). The bounds for active and inactive children are therefore both at most \(U_b(t)\), proving \(V^*(v)\le U_b(t)\). Since \(v\) is \textsc{Max}, we also have \(V^*(b)\le V^*(v)\). It remains to control the blocker's upper-endpoint error.

When the selected side is \(U\) at \(b\), Lemma~\ref{lem:capped-endpoint-control} gives
\[
U_b(t)-V^*(b)\le \rho,
\]
and the value ordering just proved yields
\[
|V^*(b)-V^*(v)|=V^*(v)-V^*(b)\le U_b(t)-V^*(b)\le \rho.
\]
If the rule instead selects side \(L\) at \(b\), the opposite side has no unresolved capped scale:
\[
K_U^{\le k}(b,t)=+\infty.
\]
Lemma~\ref{lem:capped-endpoint-control} therefore gives
\[
U_b(t)-V^*(b)\le \bar\rho/2\le \rho,
\]
which, substituted into the preceding gap inequality, again gives \(|V^*(b)-V^*(v)|\le \rho\).

\emph{Comparison at a \textsc{Min} node.} For \(v\) \textsc{Min} and \(s=U\), we bound every child value from below by the blocker's lower endpoint. Lemma~\ref{lem:capped-recursive-progress} gives \(A_U(v,k,t)\neq\varnothing\). Let \(\lambda=\lambda_U(v,t)\) and \(b=b_U(v,k,t)\). An active child \(x\in A_U(v,k,t)\) satisfies, by the definition of \(b\),
\[
V^*(x)\ge L_x(t)\ge L_b(t).
\]
For \(x\notin A_U(v,k,t)\), either
\[
L_x(t)\ge \lambda-\bar\rho/2,
\]
or \(\mathsf{Comp}_U(x,k,t)\), in which case
\[
V^*(x)\ge U_x(t)-\bar\rho/2\ge \lambda-\bar\rho/2
\]
by Lemma~\ref{lem:cert-soundness}. Since \(b\in A_U(v,k,t)\), we have
\[
L_b(t)<\lambda-\bar\rho/2.
\]
The active and inactive child bounds are therefore both at least \(L_b(t)\), which proves \(V^*(v)\ge L_b(t)\). Since \(v\) is \textsc{Min}, we also have \(V^*(b)\ge V^*(v)\). We now control the blocker's lower-endpoint error.

When the selected side is \(L\) at \(b\), Lemma~\ref{lem:capped-endpoint-control} gives
\[
V^*(b)-L_b(t)\le \rho,
\]
and the value ordering above yields
\[
|V^*(b)-V^*(v)|=V^*(b)-V^*(v)\le V^*(b)-L_b(t)\le \rho.
\]
If the rule instead selects side \(U\) at \(b\), the opposite side is complete throughout the capped range:
\[
K_L^{\le k}(b,t)=+\infty,
\]
and Lemma~\ref{lem:capped-endpoint-control} gives
\[
V^*(b)-L_b(t)\le \bar\rho/2\le \rho.
\]
Substitution into the gap inequality gives \(|V^*(b)-V^*(v)|\le\rho\), completing the comparison case.
\end{proof}

At the root, scale selection has no parent cap. The same endpoint argument applies to the uncapped active scale.

\begin{lemma}[Endpoint control at node-local active scales]\label{lem:active-scale-endpoint-control}
Assume \(\mathcal E_\delta\). Fix a round \(t\), an explored non-root node \(x\), and a side \(s\in\{L,U\}\). If
\[
k=K_s(x,t)<+\infty,
\]
then
\[
\begin{cases}
V^*(x)-L_x(t)\le \rho_k, & s=L,\\[3pt]
U_x(t)-V^*(x)\le \rho_k, & s=U.
\end{cases}
\]
If \(K_s(x,t)=+\infty\), then the corresponding one-sided error is zero.
\end{lemma}

\begin{proof}
If \(K_s(x,t)=+\infty\) because \(U_x(t)=L_x(t)\), both errors are zero by interval validity. We may therefore assume positive width and write \(\ell=\underline k(x,t)\).

For \(k=K_s(x,t)<+\infty\), the definition gives \(k\ge \ell\). If \(k=\ell\), the width bound defining \(\ell\) is
\[
U_x(t)-L_x(t)\le \rho_\ell=\rho_k,
\]
which controls either endpoint error by interval validity. If \(k>\ell\), the minimality of \(k\) in the definition of \(K_s(x,t)\) gives \(\mathsf{Comp}_s(x,k-1,t)\). Certificate soundness at this preceding scale bounds the required error by
\[
\rho_{k-1}/2=\rho_k.
\]

Finally, suppose \(K_s(x,t)=+\infty\) while the width is positive. Then \(\mathsf{Comp}_s(x,i,t)\) holds at every \(i\ge \ell\). Lemma~\ref{lem:cert-soundness} bounds the same endpoint error by \(\rho_i/2\) for all \(i\ge \ell\). Letting \(i\to\infty\) and using \(\rho_i\downarrow 0\) proves that this error is zero.
\end{proof}

The root rule compares the leader with a full contender. Its choice of the coarser unresolved scale controls the endpoint errors needed to bound the selected action's effective gap.

\begin{lemma}[Root active-call localization]\label{lem:root-active-call-localization}
Assume \(\mathcal E_\delta\) and \(\varepsilon=0\). Consider an outer-loop round \(t\), and let \(x\in\mathrm{Ch}(r)\) be the root child selected by Algorithm~\ref{alg:2ffs-t}, with selected side \(s\) and finite active scale \(k<+\infty\). Then
\[
\Delta_x^{\mathrm{eff}}\le 2\rho_k.
\]
\end{lemma}

\begin{proof}
We combine the endpoint bounds at the selected scale with the overlap that prevents root stopping. Let
\[
a\in\argmax_{y\in\mathrm{Ch}(r)}L_y(t),
\qquad
b\in\argmax_{y\in\mathrm{Ch}(r)\setminus\{a\}}U_y(t)
\]
be the current leader and challenger. The failed stopping test and \(\varepsilon=0\) give
\[
L_a(t)<U_b(t).
\]
For any root child \(y\), Definition~\ref{def:eff-gap} gives
\[
\Delta_y^{\mathrm{eff}}
=
\max\{\Delta_*,\,V^*(r,a^*)-V^*(r,y)\}.
\]
Thus \(\Delta_{a^*}^{\mathrm{eff}}=\Delta_*\), while for every nonoptimal root child \(y\neq a^*\),
\[
\Delta_y^{\mathrm{eff}}=V^*(r,a^*)-V^*(r,y),
\]
because \(V^*(r,a^*)-V^*(r,y)\ge \Delta_*\).

\emph{The leader is selected.} If the selected child is \(a\), then \(k=K_L(a,t)\), and Lemma~\ref{lem:active-scale-endpoint-control} gives
\[
V^*(r,a)-L_a(t)\le \rho_k.
\]
The root selection rule also controls both endpoints of the challenger. Its full-contender scale \(k_b\) is either \(+\infty\), in which case both \(K_L(b,t)\) and \(K_U(b,t)\) are \(+\infty\), or it is finite with \(\rho_{k_b}\le \rho_k\). In the latter case, the contender-side definition gives, for each \(s'\in\{L,U\}\), either \(K_{s'}(b,t)=+\infty\) or \(\rho_{K_{s'}(b,t)}\le \rho_{k_b}\le \rho_k\). Lemma~\ref{lem:active-scale-endpoint-control} therefore yields
\[
V^*(r,b)-L_b(t)\le \rho_k,
\qquad
U_b(t)-V^*(r,b)\le \rho_k.
\]

If \(a=a^*\), then \(\Delta_a^{\mathrm{eff}}=\Delta_*\). Since \(b\) is a competitor, the failed stopping test and the endpoint bounds give
\[
\Delta_*
\le V^*(r,a)-V^*(r,b)
<
\bigl(V^*(r,a)-L_a(t)\bigr)+\bigl(U_b(t)-V^*(r,b)\bigr)
\le 2\rho_k.
\]
If \(a\neq a^*\), then \(a^*\neq a\), so \(U_b(t)\ge U_{a^*}(t)\ge V^*(r,a^*)\). The leader ordering gives \(L_a(t)\ge L_b(t)\), while the challenger bounds give \(L_b(t)\ge V^*(r,b)-\rho_k\) and \(U_b(t)\le V^*(r,b)+\rho_k\). Combining these inequalities yields
\[
\Delta_a^{\mathrm{eff}}
=V^*(r,a^*)-V^*(r,a)
\le U_b(t)-L_a(t)
\le U_b(t)-L_b(t)
\le 2\rho_k.
\]

\emph{The challenger is selected.} Now suppose the selected child is \(b\), with finite full-contender scale \(k\). For each \(s'\in\{L,U\}\), either \(K_{s'}(b,t)=+\infty\) or \(\rho_{K_{s'}(b,t)}\le \rho_k\), by the contender-side definition. Lemma~\ref{lem:active-scale-endpoint-control} therefore bounds both endpoint errors of \(b\) by \(\rho_k\):
\[
V^*(r,b)-L_b(t)\le \rho_k,
\qquad
U_b(t)-V^*(r,b)\le \rho_k.
\]
For the leader \(a\), either \(K_L(a,t)=+\infty\), in which case the lower error is zero, or the algorithm's choice of the challenger implies \(\rho_{K_L(a,t)}\le \rho_k\). Hence
\[
V^*(r,a)-L_a(t)\le \rho_k.
\]

If \(b=a^*\), then \(a\) is a competitor. The leader ordering bounds their value gap by the challenger's lower-endpoint error:
\[
\Delta_b^{\mathrm{eff}}
=\Delta_*
\le V^*(r,b)-V^*(r,a)
\le V^*(r,b)-L_b(t)
\le \rho_k,
\]
where we used \(L_a(t)\ge L_b(t)\) and \(L_a(t)\le V^*(r,a)\). If \(b\neq a^*\) and \(a=a^*\), then the failed stop test gives
\[
\Delta_b^{\mathrm{eff}}
=V^*(r,a)-V^*(r,b)
<
\bigl(V^*(r,a)-L_a(t)\bigr)+\bigl(U_b(t)-V^*(r,b)\bigr)
\le 2\rho_k.
\]
In the remaining case, \(b\neq a^*\) and \(a\neq a^*\). Since \(a^*\neq a\), the challenger choice \(b\) gives
\[
V^*(r,a^*)\le U_{a^*}(t)\le U_b(t)\le V^*(r,b)+\rho_k.
\]
The gap is therefore bounded by the challenger's upper-endpoint error:
\[
\Delta_b^{\mathrm{eff}}=V^*(r,a^*)-V^*(r,b)\le \rho_k.
\]
This covers every choice of the selected root child.
\end{proof}

The root bound and the one-step bound combine without a loss at each depth, because capped descent never decreases the target tolerance.

\begin{proposition}[Active-call localization with explicit depth bound]\label{prop:active-call-localization}
Assume \(\mathcal E_\delta\) and \(\varepsilon=0\). Consider any invocation \(\Call{Resolve}{v,s,k}\) with finite scale \(k\) made by Algorithm~\ref{alg:2ffs-t} during an outer-loop iteration. Let \(m\) be the number of recursive edges from the initially selected root child in that outer-loop iteration to \(v\). Then \(m\le D\) and
\[
\Delta_v^{\mathrm{eff}}\le C_{\mathrm{act}}(D)\rho_k,
\qquad
C_{\mathrm{act}}(D)=2.
\]
Equivalently, every active finite-scale call satisfies
\[
\Delta_v^{\mathrm{eff}}\le 2\rho_k.
\]
\end{proposition}

\begin{proof}
We propagate the root bound along the call chain. Fix one outer-loop iteration and let
\[
x_0,x_1,\ldots,x_m=v
\]
be the chain from the selected root child to the current invocation, with sides and finite scales
\[
(s_0,k_0),(s_1,k_1),\ldots,(s_m,k_m)=(s,k).
\]
Each \(x_i\), for \(i\ge 1\), is a child of \(x_{i-1}\), so the tree depth bounds the chain length: \(m\le D\).

Let \(t_0\) be the outer-loop round at which \(x_0\) was selected. For \(i\ge 1\), let \(t_i\) be the round, within the same nested execution, at which the call \(\Call{Resolve}{x_{i-1},s_{i-1},k_{i-1}}\) calls \(\Call{Resolve}{x_i,s_i,k_i}\). Each localization lemma is applied to the intervals and active set at its own call time. The true values and effective gaps remain fixed along the chain.

At round \(t_0\), Lemma~\ref{lem:root-active-call-localization} gives
\[
\Delta_{x_0}^{\mathrm{eff}}\le 2\rho_{k_0}.
\]
For each \(i\ge 1\), apply Lemma~\ref{lem:capped-one-step-localization} at round \(t_i\) to the call from \(x_{i-1}\) to \(x_i\). This gives
\[
k_i\le k_{i-1}
\qquad\text{and}\qquad
|V^*(x_i)-V^*(x_{i-1})|\le \rho_{k_i}.
\]
The cap is essential here: by \(\rho_{j+1}=\rho_j/2\), the child tolerance is at least the parent tolerance,
\[
\rho_{k_{i-1}}\le \rho_{k_i}.
\]
These two bounds let us carry the same constant through the maximum in the effective-gap recursion. Specifically, we induct on \(i\) to show
\[
\Delta_{x_i}^{\mathrm{eff}}\le 2\rho_{k_i}.
\]
The case \(i=0\) is the root bound above. For the induction step, Definition~\ref{def:eff-gap} gives
\[
\Delta_{x_i}^{\mathrm{eff}}
=
\max\!\left\{
\Delta_{x_{i-1}}^{\mathrm{eff}},
|V^*(x_i)-V^*(x_{i-1})|
\right\}.
\]
The induction hypothesis controls the inherited gap, and the one-step bound controls the new edge. Hence
\[
\Delta_{x_i}^{\mathrm{eff}}
\le
\max\!\left\{
2\rho_{k_{i-1}},
\rho_{k_i}
\right\}
\le
2\rho_{k_i}.
\]
Taking \(i=m\) proves the claim for \(\Call{Resolve}{v,s,k}\). Thus \(C_{\mathrm{act}}(D)=2\): the constant is independent of \(D\), and hence polynomial in \(D\).
\end{proof}

\begin{lemma}[Local scale completion]\label{lem:local-scale-completion}
Fix a round \(t\), an explored non-root node \(v\), and a scale \(k\ge 0\). Assume that the local counter \(M_{v,k}^{\mathrm{loc}}(t)\) is defined. On \(\mathcal E_\delta\), if
\[
M_{v,k}^{\mathrm{loc}}(t)\ge \Gamma_v^{(k)}(\delta_v),
\]
then
\[
\mathrm{width}\bigl(I_v^{\mathrm{loc}}(t)\bigr)\le \rho_k/2.
\]
Consequently, \(\mathsf{Cert}_L(v,k,t)\) and \(\mathsf{Cert}_U(v,k,t)\) both hold.
\end{lemma}

\begin{proof}
We bound the local interval using the oracle that determines the local cost. If \(B(h(v))\le \rho_k/4\), then Eq.~\eqref{eq:gamma-rho} gives \(\Gamma_v^{(k)}(\delta_v)=0\). In this case, the fast interval is already sufficient: since
\[
I_v^{\mathrm{loc}}(t)=I_v^F\cap I_v^S(t)\subseteq I_v^F,
\]
Eq.~\eqref{eq:I-fast} gives
\[
\mathrm{width}\bigl(I_v^{\mathrm{loc}}(t)\bigr)
\le 2B(h(v))
\le \rho_k/2.
\]

Suppose now that \(B(h(v))>\rho_k/4\). Let \(n_{v,k}(t)\) count the slow samples issued at \(v\) while resolving scale \(k\), through round \(t\). The definition of the local counter gives
\[
M_{v,k}^{\mathrm{loc}}(t)=c\,n_{v,k}(t).
\]
The assumed spending bound and Eq.~\eqref{eq:gamma-rho} therefore imply
\[
n_{v,k}(t)\ge m_v(\rho_k,\delta_v).
\]
The total sample count at \(v\) includes these samples, so
\[
N_v^S(t)\ge n_{v,k}(t)\ge m_v(\rho_k,\delta_v).
\]
Since \(n\mapsto \beta_v(n,\delta_v)\) is nonincreasing,
\[
\beta_v\!\bigl(N_v^S(t),\delta_v\bigr)
\le
\beta_v\!\bigl(m_v(\rho_k,\delta_v),\delta_v\bigr)
\le \rho_k/4.
\]
The running intersection in Eq.~\eqref{eq:slow-interval} includes this current confidence interval. Hence \(I_v^S(t)\) satisfies
\[
\mathrm{width}\bigl(I_v^S(t)\bigr)\le \rho_k/2.
\]
The inclusion \(I_v^{\mathrm{loc}}(t)\subseteq I_v^S(t)\) proves the required local-width bound. In either case, this is the local-width alternative in \(\mathsf{Cert}_s(v,k,t)\), for both \(s=L\) and \(s=U\).
\end{proof}

\begin{lemma}[Budgeted recursive spending]\label{lem:budgeted-recursive-spending}
Fix any execution of Algorithm~\ref{alg:2ffs-t}. For every invocation \(\Call{Resolve}{v,s,k,B}\), if the invocation returns \((q,z)\), then
\[
0\le q\le B
\]
with the convention that the inequality is void when \(B=+\infty\). Moreover, if \(z=\mathsf{blocked}\), then \(q=0\). Finally, for every explored non-root node \(v\), scale \(k\), and round \(t\),
\[
M_{v,k}^{\mathrm{exp}}(t)
\le
\mathsf B_{v,k}^{\mathrm{rec}}
=
\alpha_{h(v)}\Gamma_v^{(k)}(\delta_v).
\]
Here an undefined scale-local counter is interpreted as zero.
\end{lemma}

\begin{proof}
We first show that every return respects the input cap and that a blocked call incurs no cost, by induction on the remaining depth \(h(v)\). The certificate branch returns \((0,\mathsf{cert})\). In the local branch, \(\Call{LocalStep}{v,k,B}\) returns \((0,\mathsf{blocked})\) if \(B<c\); otherwise it takes one slow sample and returns \((c,\mathsf{progress})\), with \(B\ge c\). Thus both properties hold for the local branch: its cost is at most \(B\), and it spends nothing when blocked. This also proves the leaf case.

For an internal node, let
\[
R=\mathsf B_{v,k}^{\mathrm{rec}}-M_{v,k}^{\mathrm{exp}}.
\]
If \(R\le 0\), the call uses the local branch just considered. Otherwise, its child-work cap is \(B_{\mathrm{rec}}=\min\{B,R\}\). For an unexpanded \(v\), the comparison with the expansion cost determines what happens. If \(|\mathrm{Ch}(v)|>R\), expansion exceeds the node's own budget, so the call uses the local branch. If \(|\mathrm{Ch}(v)|\le R\) but \(|\mathrm{Ch}(v)|>B\), expansion fits that budget but exceeds the inherited cap, and the call returns \((0,\mathsf{blocked})\). In the remaining case,
\[
|\mathrm{Ch}(v)|\le B_{\mathrm{rec}}\le B,
\]
so expansion respects the cap.

If \(v\) is already expanded, there is at most one child call, with cap \(B_{\mathrm{rec}}\). The induction hypothesis bounds its returned cost \(q\) by \(B_{\mathrm{rec}}\le B\), and makes that cost zero if the child reports \(\mathsf{blocked}\). When this happens and the inherited cap is binding \((B<R)\), the parent also returns \((0,\mathsf{blocked})\). Otherwise it uses the local branch, which respects the cap and costs zero exactly when blocked. If the child is not blocked, the parent simply returns the child's result. This completes the induction on the return properties.

It remains to bound \(M_{v,k}^{\mathrm{exp}}\). This counter starts at zero, with undefined counters also read as zero, and has only two possible updates. When \(v\) is expanded at scale \(k\), the algorithm first checks
\[
|\mathrm{Ch}(v)|\le B_{\mathrm{rec}}\le R
=\mathsf B_{v,k}^{\mathrm{rec}}-M_{v,k}^{\mathrm{exp}},
\]
so adding the expansion cost leaves \(M_{v,k}^{\mathrm{exp}}\) at most \(\mathsf B_{v,k}^{\mathrm{rec}}\). When a child call returns \((q,z)\) with \(z\neq\mathsf{blocked}\), the parent adds to \(M_{v,k}^{\mathrm{exp}}\) the cost \(q\). The return property just proved gives
\[
q\le B_{\mathrm{rec}}\le R
=\mathsf B_{v,k}^{\mathrm{rec}}-M_{v,k}^{\mathrm{exp}},
\]
so this update also leaves the counter at most \(\mathsf B_{v,k}^{\mathrm{rec}}\). Since no other operation increases \(M_{v,k}^{\mathrm{exp}}\), the bound holds throughout the execution.
\end{proof}

\begin{lemma}[Aggregate obligation accounting]\label{lem:aggregate-obligation-accounting}
Assume \(\mathcal E_\delta\). Fix an explored non-root node \(v\) and scale \(k\). Undefined scale-local counters are interpreted as zero. For every round \(t\),
\[
M_{v,k}^{\mathrm{loc}}(t)\le \Gamma_v^{(k)}(\delta_v),
\qquad
M_{v,k}^{\mathrm{exp}}(t)
\le
\alpha_{h(v)}\Gamma_v^{(k)}(\delta_v).
\]
Consequently the total work ever charged to the scale-\(k\) obligation at \(v\), aggregated over both sides, is bounded by
\[
M_{v,k}^{\mathrm{loc}}(t)+M_{v,k}^{\mathrm{exp}}(t)
\le
\bigl(1+\alpha_{h(v)}\bigr)\Gamma_v^{(k)}(\delta_v).
\]
Moreover, once \(M_{v,k}^{\mathrm{loc}}(t)\) reaches \(\Gamma_v^{(k)}(\delta_v)\), both side predicates \(\mathsf{Comp}_L(v,k,\cdot)\) and \(\mathsf{Comp}_U(v,k,\cdot)\) remain true thereafter, and every later invocation of \(\Call{Resolve}{v,s,k,\cdot}\) for either side returns before doing positive work.
\end{lemma}

\begin{proof}
Lemma~\ref{lem:budgeted-recursive-spending} gives the recursive bound. For the local bound, we show that completion prevents the counter from passing its prescribed limit. The counter starts at zero, and only \(\Call{LocalStep}{v,k,B}\) increases it, by exactly \(c\) each time. If \(\Gamma_v^{(k)}(\delta_v)=0\), Lemma~\ref{lem:local-scale-completion} gives both side certificates before any slow query at \(v\) at scale \(k\). Thus \(\Call{Resolve}{v,s,k,\cdot}\) returns at its first branch and \(M_{v,k}^{\mathrm{loc}}\) stays zero.

If \(\Gamma_v^{(k)}(\delta_v)>0\), the identity \(\Gamma_v^{(k)}(\delta_v)= \Gamma_v(\rho_k,\delta_v)\) shows that the limit is \(c\,m_v(\rho_k,\delta_v)\). Thus \(M_{v,k}^{\mathrm{loc}}\) cannot skip the limit: its values lie in \(\{0,c,2c,\ldots,\Gamma_v^{(k)}(\delta_v)\}\) until completion. Upon first reaching \(\Gamma_v^{(k)}(\delta_v)\), Lemma~\ref{lem:local-scale-completion} gives
\[
\mathsf{Cert}_L(v,k,\cdot)
\qquad\text{and}\qquad
\mathsf{Cert}_U(v,k,\cdot).
\]
The ensuing \(\Call{RefreshAndLatch}{k}\) in Algorithm~\ref{alg:2ffs-actions} sets both \(\mathsf{Done}\) flags. Therefore \(\mathsf{Comp}_L(v,k,\cdot)\) and \(\mathsf{Comp}_U(v,k,\cdot)\) remain true. By Lemma~\ref{lem:no-same-scale-rework}, every subsequent call at this node and scale returns without positive work, on either side. The local counter therefore stays at most \(\Gamma_v^{(k)}(\delta_v)\).

Adding the two counter bounds proves the aggregate estimate. Persistence follows from the same latching argument and Lemma~\ref{lem:no-same-scale-rework}.
\end{proof}

\begin{lemma}[Aggregate local summation]\label{lem:aggregate-local-summation}
Given \emph{Assumption~\ref{ass:local-scale-regularity}}, for a non-root node \(v\), let
\[
\mathcal K_v^{\mathrm{act}}
\coloneqq
\{k\ge 0:\ \Delta_v^{\mathrm{eff}}\le 2\rho_k\}.
\]
Then, we have
\[
\sum_{k\in\mathcal K_v^{\mathrm{act}}}
\Gamma_v^{(k)}(\delta_v)
\le
\Lambda_{\mathrm{loc}}\,
\Gamma_v(\Delta_v^{\mathrm{eff}},\delta_v).
\]
Moreover, for every \(x\ge 0\) and every finite prefix of scales considered in an execution,
\[
\sum_{\substack{k\ge 0:\\ \Gamma_v^{(k)}(\delta_v)\le x}}
\Gamma_v^{(k)}(\delta_v)
\le
\Lambda_{\mathrm{pre}}\,x,
\]
where the sum is restricted to that finite prefix.
\end{lemma}

\begin{proof}
For each bound, we identify a prefix of scales and apply prefix regularity. The sequence \(k\mapsto \Gamma_v^{(k)}(\delta_v)\) is nondecreasing: as \(\rho_k\) decreases, the sample requirement cannot decrease, since \(m_v(\rho,\delta_v)\) is nonincreasing in the target radius \(\rho\).

The first claim is immediate if \(\mathcal K_v^{\mathrm{act}}=\varnothing\). Otherwise, \(\mathcal K_v^{\mathrm{act}}\) is the prefix \(\{0,\ldots,K_v\}\). Its last index \(K_v\) satisfies
\[
\rho_{K_v}\ge \frac{\Delta_v^{\mathrm{eff}}}{2}.
\]
This lets us compare the last cost with the cost at the effective gap. Monotonicity of \(\Gamma_v(\rho,\delta_v)\) in \(\rho\) and Eq.~\eqref{eq:local-gap-regularity} give
\[
\Gamma_v^{(K_v)}(\delta_v)
=
\Gamma_v(\rho_{K_v},\delta_v)
\le
\Gamma_v\!\left(\frac{\Delta_v^{\mathrm{eff}}}{2},\delta_v\right)
\le
\Lambda_{\mathrm{gap}}\Gamma_v(\Delta_v^{\mathrm{eff}},\delta_v).
\]
The prefix condition \eqref{eq:local-prefix-regularity} at \(K_v\) now bounds the sum by this last cost:
\[
\sum_{k\in\mathcal K_v^{\mathrm{act}}}\Gamma_v^{(k)}(\delta_v)
\le
\Lambda_{\mathrm{pre}}\Gamma_v^{(K_v)}(\delta_v)
\le
\Lambda_{\mathrm{loc}}\Gamma_v(\Delta_v^{\mathrm{eff}},\delta_v).
\]

For the second claim, fix the finite prefix \(\{0,\ldots,K_{\max}\}\) of scales in question. If no scale satisfies the cost threshold, the sum is zero. Otherwise, monotonicity makes the admissible set a prefix \(\{0,\ldots,K_x\}\), whose last cost satisfies \(\Gamma_v^{(K_x)}(\delta_v)\le x\). Applying prefix regularity once more,
\[
\sum_{k=0}^{K_x}\Gamma_v^{(k)}(\delta_v)
\le
\Lambda_{\mathrm{pre}}\Gamma_v^{(K_x)}(\delta_v)
\le
\Lambda_{\mathrm{pre}}x.
\]
\end{proof}

We now compare the algorithm's total spending with the recursive reference cost. This comparison must also account for the cap needed to make further progress: a child may be blocked even though that cap has not yet been spent. The next definition and lemma make this distinction precise.

\begin{definition}[Aggregate subtree call cost]\label{def:aggregate-subtree-call-cost}
For a non-root node \(v\), let \(\mathfrak C_v(t)\) be the sum of the costs returned by all invocations at \(v\) in the first \(t\) completed outer-loop rounds. Equivalently, this is the slow cost at \(v\), together with all fast and slow costs strictly below \(v\); it excludes the fast query that first exposed \(v\). Set this quantity to zero for an unexposed node. Write \(L_v^{\mathrm{cost}}(t)\) for the direct slow cost at \(v\), and \(E_v^{\mathrm{cost}}(t)\) for the remaining, recursive cost. Then
\[
\mathfrak C_v(t)=L_v^{\mathrm{cost}}(t)+E_v^{\mathrm{cost}}(t),
\qquad
E_v^{\mathrm{cost}}(t)
=\mathbf 1\{v\text{ expanded by }t\}|\mathrm{Ch}(v)|
+\sum_{u\in\mathrm{Ch}(v)}\mathfrak C_u(t).
\]
The second term is zero at a leaf. Costs in a child subtree also appear in its ancestors' costs, but no sum over all nodes is used: the root-child subtrees are disjoint, and their sum counts every post-initialization oracle cost exactly once.
\end{definition}

\begin{lemma}[Minimum cap for an unresolved obligation]\label{lem:minimum-progress-cap}
Assume \(\mathcal E_\delta\). At a fixed state of the algorithm, let \((v,s,k)\) be a finite unresolved obligation at an explored node, so that \(\neg\mathsf{Comp}_s(v,k,t)\). There is a finite positive number \(d_v(s,k;t)\), called its \emph{minimum progress cap}, such that \(\Call{Resolve}{v,s,k,B}\) returns \((0,\mathsf{blocked})\) exactly when \(B<d_v(s,k;t)\), and returns positive cost when \(B\ge d_v(s,k;t)\). Fix the algorithm's tie-breaking rules when defining this quantity. For an internal node set
\[
R=\alpha_{h(v)}\Gamma_v^{(k)}(\delta_v)-M_{v,k}^{\mathrm{exp}}(t),
\qquad b_v=|\mathrm{Ch}(v)|.
\]
The cap has the following recursive description:
\begin{enumerate}
\item If \(v\) is a leaf, \(R\le0\), or \(v\) is unexpanded with \(b_v>R\), then \(d_v(s,k;t)=c\).
\item If \(v\) is unexpanded and \(b_v\le R\), then \(d_v(s,k;t)=b_v\).
\item Otherwise \(v\) is expanded and \(R>0\). Let \((u,s_u,k_u)=\Call{ChildObligation}{v,s,k}\), and write \(d_u=d_u(s_u,k_u;t)\). Then
\[
d_v(s,k;t)=
\begin{cases}
d_u,&d_u\le R,\\
\max\{R,c\},&d_u>R.
\end{cases}
\]
\end{enumerate}
In particular, at every internal node,
\[
d_v(s,k;t)\le\max\{R,c\}.
\]
Moreover, if \(v\) is expanded and \(\Delta_v^{\mathrm{eff}}\le2\rho_k\), the child obligation selected by \(\Call{ChildObligation}{v,s,k}\) satisfies \(\Delta_u^{\mathrm{eff}}\le2\rho_{k_u}\), regardless of the value of \(R\) and whether the resolver actually reaches its child-call branch.
\end{lemma}

\begin{proof}
We prove the cap formula by induction on remaining depth. The local and expansion branches of Algorithm~\ref{alg:2ffs-resolve} give the first two cases directly. In the third case, Lemma~\ref{lem:capped-recursive-progress} supplies a finite unresolved child obligation. Certificate evaluations do not affect this argument: at fixed intervals, a bottom-up induction shows that latching a true certificate leaves all current completion predicates unchanged. Hence these zero-cost evaluations preserve both the selected obligation and its cap threshold.

If \(d_u\le R\), the child receives enough cap exactly when \(B\ge d_u\). If \(B<d_u\), then \(B<R\), so the parent passes the blocked return upward. If \(d_u>R\), the child is blocked for every cap \(\min\{B,R\}\). The parent can then take a local step exactly when \(B\ge R\) and \(B\ge c\). These are the two alternatives in the claimed formula. They also show that the threshold is finite and bounded by \(\max\{R,c\}\).

The child-selection rule and Lemma~\ref{lem:capped-one-step-localization} give \(k_u\le k\) and \(|V^*(u)-V^*(v)|\le\rho_{k_u}\). These bounds depend only on the current state and the child-selection rule; they do not require the child call to receive enough cap to perform work. Thus
\[
\Delta_u^{\mathrm{eff}}
=\max\{\Delta_v^{\mathrm{eff}},|V^*(u)-V^*(v)|\}
\le\max\{2\rho_k,\rho_{k_u}\}\le2\rho_{k_u}.
\]
\end{proof}

\begin{proposition}[Aggregate subtree charging]\label{prop:aggregate-subtree-charging}
Assume \(\mathcal E_\delta\), \(\varepsilon=0\), and Assumption~\ref{ass:local-scale-regularity}, and let \(P_h\) be defined in Eq.~\eqref{eq:P-depth-recursion}. For every finite round \(T\) and every non-root node \(v\),
\[
\mathfrak C_v(T)\le P_{h(v)}J_v^*(\boldsymbol\delta).
\]
In fact, for every currently unresolved finite obligation \((v,s,k)\) with \(\Delta_v^{\mathrm{eff}}\le2\rho_k\), the stronger bound holds:
\begin{equation}\label{eq:subtree-cost-plus-cap}
\mathfrak C_v(T)+d_v(s,k;T)
\le P_{h(v)}J_v^*(\boldsymbol\delta).
\end{equation}
Consequently,
\[
C_T\le|\mathrm{Ch}(r)|+
P_{D-1}\sum_{a\in\mathrm{Ch}(r)}J_a^*(\boldsymbol\delta)
\le P_DH(\boldsymbol\delta).
\]
For \(\alpha_h=(h+1)^2\),
\[
P_D\le\Lambda_{\mathrm{loc}}(1+\alpha_D)
\prod_{i=1}^D\left(1+\frac{\Lambda_{\mathrm{pre}}}{\alpha_i}\right),
\]
so \(P_D=O_{\Lambda_{\mathrm{pre}},\Lambda_{\mathrm{gap}}}(D^2)\).
\end{proposition}

\begin{proof}
We use induction on remaining depth to prove the cost bound together with \eqref{eq:subtree-cost-plus-cap}. The stronger claim is needed when a child call is blocked: past spending alone does not control the cap required for its next action. We prove both claims for every finite execution prefix and every current unresolved obligation satisfying the localization condition. Throughout, we consider states before new positive work and interpret undefined counters as zero. A blocked call changes neither intervals nor spending counters.

\emph{The local bound.} We first bound spending by the cost of local certification. Each positive cost returned at scale \(j\) increases exactly one of \(M_{v,j}^{\mathrm{loc}}\) and \(M_{v,j}^{\mathrm{exp}}\) by that cost. Proposition~\ref{prop:active-call-localization} restricts every such scale to \(\mathcal K_v^{\mathrm{act}}\). Therefore Lemma~\ref{lem:aggregate-obligation-accounting} gives
\begin{equation}\label{eq:local-route-aggregate-bound}
\mathfrak C_v(T)
\le(1+\alpha_{h(v)})
\sum_{j\in\mathcal K_v^{\mathrm{act}}}\Gamma_v^{(j)}(\delta_v).
\end{equation}
To include the cap for a current unresolved obligation \((v,s,k)\), observe that \(k\in\mathcal K_v^{\mathrm{act}}\). Local scale completion then gives
\[
\Gamma_v^{(k)}(\delta_v)>0,
\qquad
M_{v,k}^{\mathrm{loc}}(T)\le\Gamma_v^{(k)}(\delta_v)-c.
\]
The subtraction is justified because both quantities are multiples of \(c\). For an internal node, write \(R=\alpha_{h(v)}\Gamma_v^{(k)}-M_{v,k}^{\mathrm{exp}}\ge0\). Lemma~\ref{lem:minimum-progress-cap} then gives
\[
\begin{split}
M_{v,k}^{\mathrm{loc}}+M_{v,k}^{\mathrm{exp}}+d_v(s,k;T)
&\le M_{v,k}^{\mathrm{loc}}+M_{v,k}^{\mathrm{exp}}+\max\{R,c\}\\
&\le (1+\alpha_{h(v)})\Gamma_v^{(k)}.
\end{split}
\]
The unused allowance at this scale therefore covers the additional cap: the right-hand side of \eqref{eq:local-route-aggregate-bound} also bounds \(\mathfrak C_v(T)+d_v(s,k;T)\). By Lemma~\ref{lem:aggregate-local-summation}, both are at most \((1+\alpha_{h(v)})\Lambda_{\mathrm{loc}} \Gamma_v(\Delta_v^{\mathrm{eff}},\delta_v)\). At a leaf, there is no recursive spending, so the bound reduces to \(\Lambda_{\mathrm{loc}}\Gamma_v(\Delta_v^{\mathrm{eff}},\delta_v)\). If an unresolved obligation exists, the unused local sample covers its cap \(c\). This proves the base case. It also proves both claims at an internal node whenever the local term attains the minimum defining \(J_v^*\).

\emph{The recursive bound.} It remains to consider a node \(v\) of remaining depth \(h\ge1\) for which the recursive term attains the minimum. Set
\[
Q_v\coloneqq P_{h-1}
\left(|\mathrm{Ch}(v)|+\sum_{u\in\mathrm{Ch}(v)}J_u^*(\boldsymbol\delta)\right)
=P_{h-1}J_v^*(\boldsymbol\delta).
\]
The cost induction hypothesis and \(P_{h-1}\ge1\) imply, at every earlier state \(t\le T\),
\begin{equation}\label{eq:R-alg-bound}
E_v^{\mathrm{cost}}(t)\le Q_v.
\end{equation}
Suppose \(v\) is expanded and \((v,s,k)\) is a current unresolved obligation satisfying the localization condition. Let \((u,s_u,k_u)\) be the child obligation given by Lemma~\ref{lem:capped-recursive-progress}. It exists even if the resolver has no recursive budget left, and it satisfies the localization condition by Lemma~\ref{lem:minimum-progress-cap}. Apply the stronger induction hypothesis to this child and the ordinary cost bounds to all other children. This gives
\begin{equation}\label{eq:recursive-cost-plus-child-cap}
E_v^{\mathrm{cost}}(t)+d_u(s_u,k_u;t)\le Q_v.
\end{equation}
This is the estimate that controls a blocked child call: it includes the cap needed for progress, even though that cap has not been spent.

We next bound the direct slow cost. We show that every scale \(k\) at which a slow sample is issued at \(v\) satisfies
\begin{equation}\label{eq:leakage-scale-threshold}
\alpha_h\Gamma_v^{(k)}(\delta_v)<Q_v.
\end{equation}
Consider the state just before such a sample. By Proposition~\ref{prop:active-call-localization}, its obligation satisfies the localization condition. If \(v\) is unexpanded, local sampling requires \(\alpha_h\Gamma_v^{(k)}<|\mathrm{Ch}(v)|\le Q_v\), since its recursive counter is still zero. If \(v\) is expanded and its remaining budget \(R\) is zero, then
\[
\alpha_h\Gamma_v^{(k)}=M_{v,k}^{\mathrm{exp}}
\le E_v^{\mathrm{cost}}
<E_v^{\mathrm{cost}}+d_u\le Q_v.
\]
If the node is expanded and \(R>0\), the resolver can fall back to local sampling only after a blocked child call, and only when its inherited cap is at least \(R\). The child therefore received cap \(R\), so Lemma~\ref{lem:minimum-progress-cap} gives \(d_u>R\). Consequently,
\[
\alpha_h\Gamma_v^{(k)}
=M_{v,k}^{\mathrm{exp}}+R
<M_{v,k}^{\mathrm{exp}}+d_u
\le E_v^{\mathrm{cost}}+d_u\le Q_v.
\]
Thus \eqref{eq:leakage-scale-threshold} holds whenever local sampling occurs. The local cost at each such scale is bounded by its local stopping cost. Summing these bounds with Lemma~\ref{lem:aggregate-local-summation} yields
\begin{equation}\label{eq:recursive-local-leakage}
L_v^{\mathrm{cost}}(T)
\le\sum_{\{j:\,\Gamma_v^{(j)}<Q_v/\alpha_h\}}\Gamma_v^{(j)}
\le\frac{\Lambda_{\mathrm{pre}}}{\alpha_h}Q_v.
\end{equation}
Here the sum is taken over a finite prefix containing all scales used through \(T\). Together with \eqref{eq:R-alg-bound}, this proves the ordinary cost bound with factor \((1+\Lambda_{\mathrm{pre}}/\alpha_h)P_{h-1}\).

\emph{Including the cap for further progress.} We must still include the current cap \(d=d_v(s,k;T)\). We show that it fits either the recursive bound or the unused local allowance. If \(v\) is unexpanded and the expansion fits its recursive budget, then \(d=|\mathrm{Ch}(v)|\), \(E_v^{\mathrm{cost}}=0\), and \(E_v^{\mathrm{cost}}+d\le Q_v\). If \(v\) is expanded and its child's cap fits the remaining budget, then \(d=d_u\) and \eqref{eq:recursive-cost-plus-child-cap} gives the same conclusion. Finally, if \(d_u>R\ge c\), then \(d=R<d_u\), so again \(E_v^{\mathrm{cost}}+d\le Q_v\). In each case, the recursive bound covers the cap, and \eqref{eq:recursive-local-leakage} covers the local spending.

In the remaining cases, \(d=c\) and the next positive step is local sampling: either expansion exceeds the unexpanded node's budget, \(R=0\), or \(d_u>R\) and \(R<c\). The argument for \eqref{eq:leakage-scale-threshold} applies to this scale as well, even if no sample has yet been issued there. Moreover, \(M_{v,k}^{\mathrm{loc}}\le\Gamma_v^{(k)}-c\), since the current obligation is unresolved. One additional local sample therefore leaves the direct local cost at this scale at most \(\Gamma_v^{(k)}\). Applying the threshold sum to the previously sampled scales and the current scale gives
\[
L_v^{\mathrm{cost}}(T)+d
\le\frac{\Lambda_{\mathrm{pre}}}{\alpha_h}Q_v.
\]
Together with \(E_v^{\mathrm{cost}}(T)\le Q_v\), this proves in every case
\[
\mathfrak C_v(T)+d_v(s,k;T)
\le\left(1+\frac{\Lambda_{\mathrm{pre}}}{\alpha_h}\right)Q_v
\le P_hJ_v^*(\boldsymbol\delta).
\]
This completes the simultaneous induction.

\emph{Summing at the root.} The root's only direct cost is the initialization cost \(|\mathrm{Ch}(r)|\), and every subsequent cost belongs to exactly one root-child subtree. Since \(P_h\ge P_{h-1}\ge1\) follows directly from Eq.~\eqref{eq:P-depth-recursion}, the claimed root bounds follow. Finally, unrolling that recursion and using the nondecreasing sequence \(\alpha_h\) gives the stated product bound. For \(\alpha_h=(h+1)^2\), its product is at most
\[
\exp\!\left(\Lambda_{\mathrm{pre}}
\sum_{i=1}^{\infty}\frac{1}{(i+1)^2}\right),
\]
while its remaining factor is \(O(D^2)\).
\end{proof}

\begin{lemma}[Positive progress of unresolved selected calls]\label{lem:budgeted-positive-progress}
Let
\[
q_{\min}\coloneqq \min\{c,1\}>0.
\]
Consider any invocation \(\Call{Resolve}{v,s,k,B}\) with finite scale \(k\) that begins at a round \(t\) with
\[
\neg\mathsf{Comp}_s(v,k,t).
\]
Then the invocation either returns \((0,\mathsf{blocked})\), or returns a pair \((q,z)\) with \(q\ge q_{\min}\). In particular, if \(B=+\infty\), then the invocation returns cost at least \(q_{\min}\).
\end{lemma}

\begin{proof}
We argue by induction on the remaining depth \(h(v)\). At a leaf, \(h(v)=0\), the first guard in \(\Call{Resolve}{v,s,k,B}\) is false because the obligation is unresolved. The procedure therefore calls \(\Call{LocalStep}{v,k,B}\), which returns \((0,\mathsf{blocked})\) if \(B<c\), and otherwise takes one slow sample at cost \(c\ge q_{\min}\).

Suppose \(h(v)\ge 1\) and the claim holds for the children of \(v\). The first guard is again false. A nonpositive remaining recursive budget leads directly to \(\Call{LocalStep}{v,k,B}\), whose behavior was just established. If \(v\) is unexpanded and \(|\mathrm{Ch}(v)|>R\), the same fallback \(\Call{LocalStep}{v,k,B}\) applies. For an unexpanded node in the remaining case, \(|\mathrm{Ch}(v)|\le R\): the expansion is affordable from the node's own budget. If \(|\mathrm{Ch}(v)|>B\), the inherited cap blocks it and the invocation returns \((0,\mathsf{blocked})\). If this cap also permits the expansion, the procedure expands \(v\) at cost \(|\mathrm{Ch}(v)|\ge 1\ge q_{\min}\).

Only the expanded-node case with positive recursive budget requires the induction hypothesis. Lemma~\ref{lem:capped-recursive-progress} ensures that \(\mathsf{Comp}_s(v,k,t)\) being false provides a finite capped child obligation. In a selector case, the chosen child \((u,s,k_u)\) satisfies
\[
k_u=K_s^{\le k}(u,t)<+\infty
\qquad\text{and}\qquad
\neg\mathsf{Comp}_s(u,k_u,t).
\]
In a comparison case, the blocker is called on the opposite side \(\bar s\) when a finite capped obligation remains there, and otherwise on side \(s\). The definition of \(K_{\cdot}^{\le k}\) ensures that either call targets an unresolved obligation, so the induction hypothesis applies. A positive child cost is returned unchanged. After a child return of \((0,\mathsf{blocked})\), the parent either also returns \((0,\mathsf{blocked})\) because the inherited cap is binding, or calls \(\Call{LocalStep}{v,k,B}\). That local step is itself blocked or costs at least \(q_{\min}\), completing the induction.

When \(B=+\infty\), the inherited cap cannot block progress and \(\Call{LocalStep}{v,k,B}\) is always affordable. The invocation therefore returns cost at least \(q_{\min}\).
\end{proof}

\begin{lemma}[Non-stopping rounds select a finite unresolved side]\label{lem:finite-selected-side}
Assume \(\mathcal E_\delta\) and \(\varepsilon=0\). In every outer-loop round before stopping, Algorithm~\ref{alg:2ffs-t} selects a root child \(x\), side \(s\), and scale \(k<+\infty\) such that
\[
\neg\mathsf{Comp}_s(x,k,t).
\]
\end{lemma}

\begin{proof}
Let \(a\) and \(b\) be the current leader and challenger, with respective scales \(k_a=K_L(a,t)\) and \(k_b=K_{\pm}(b,t)\). The selection rule prefers a finite scale: if \(k_a<+\infty\) and \(k_b=+\infty\), it selects the leader at scale \(k_a\); if \(k_a=+\infty\) and \(k_b<+\infty\), it selects the challenger at scale \(k_b\); and if both scales are finite, it selects the coarser one. Thus an infinite selected scale would require
\[
K_L(a,t)=+\infty
\qquad\text{and}\qquad
K_{\pm}(b,t)=+\infty.
\]
By the definition of \(K_{\pm}\), the second equality means
\[
K_L(b,t)=K_U(b,t)=+\infty.
\]

These equalities are incompatible with a failed stopping test. Indeed, if \(K_L(a,t)=+\infty\), Lemma~\ref{lem:active-scale-endpoint-control} gives
\[
V^*(r,a)-L_a(t)=0.
\]
Likewise, \(K_L(b,t)=K_U(b,t)=+\infty\) gives
\[
V^*(r,b)-L_b(t)=0,
\qquad
U_b(t)-V^*(r,b)=0.
\]
Failure of the stopping test means
\[
L_a(t)<U_b(t).
\]
Together, these relations imply \(V^*(r,a)<V^*(r,b)\), and hence
\[
L_b(t)=V^*(r,b)>V^*(r,a)=L_a(t),
\]
contradicting the choice of \(a\) as a leader with maximal lower endpoint. The selected scale must therefore be finite. By the definitions of \(K_L\), \(K_U\), and \(K_{\pm}\), every finite selected scale is unresolved, giving the required \(\neg\mathsf{Comp}\) condition.
\end{proof}

\begin{proposition}[Finite stopping and correctness]\label{prop:budgeted-finite-stopping}
Assume \(\mathcal E_\delta\), \(\varepsilon=0\), and Assumption~\ref{ass:local-scale-regularity}. Then Algorithm~\ref{alg:2ffs-t} halts after finitely many oracle calls. At its stopping time \(\tau\), every returned action \(\hat a_\tau\) is the unique optimal root action \(a^*\).
\end{proposition}

\begin{proof}
We combine the uniform bound on finite execution prefixes with a positive cost in every non-stopping round. Consistency of the confidence radii gives \(\Gamma_v(\Delta_v^{\mathrm{eff}},\delta_v)<\infty\) at every non-root node \(v\), since \(\Delta_v^{\mathrm{eff}}\ge \Delta_*>0\). It follows that \(J_v^*(\boldsymbol{\delta})<\infty\) for every non-root \(v\), by induction on \(h(v)\), and the tree's finiteness then gives \(H(\boldsymbol{\delta})<\infty\). Proposition~\ref{prop:aggregate-subtree-charging} bounds every finite prefix:
\[
C_T\le P_DH(\boldsymbol{\delta})
\qquad\text{for every finite round }T.
\]

Suppose the algorithm never stops. In each outer-loop round, Lemma~\ref{lem:finite-selected-side} supplies a finite unresolved call \(\Call{Resolve}{x,s,k,+\infty}\). The finite scale searches and finite recursion depth ensure that each round completes. By Lemma~\ref{lem:budgeted-positive-progress}, its cost is at least \(q_{\min}>0\). After \(N\) non-stopping rounds, the total cost is therefore at least the initialization cost plus \(Nq_{\min}\). This diverges as \(N\to\infty\), contradicting the uniform bound \(P_DH(\boldsymbol{\delta})\). The algorithm must halt after finitely many oracle calls.

At the stopping time \(\tau\), the returned action \(\hat a_\tau\) satisfies
\[
L_{\hat a_\tau}(\tau)
\ge
\max_{a\neq \hat a_\tau}U_a(\tau).
\]
Lemma~\ref{lem:interval-validity} makes every root-child interval valid on \(\mathcal E_\delta\), so
\[
V^*(r,\hat a_\tau)
\ge
L_{\hat a_\tau}(\tau)
\ge
\max_{a\neq \hat a_\tau}U_a(\tau)
\ge
\max_{a\neq \hat a_\tau}V^*(r,a).
\]
Thus \(\hat a_\tau\) is optimal, and uniqueness gives \(\hat a_\tau=a^*\).
\end{proof}
\subsection{Proof of Theorem~\ref{thm:budgeted-general-depth}}\label{app:pf3.6}

We now assemble the bounds established in Appendix~\ref{app:general-depth-proofs}.

\begin{proof}
When $\rho_0=0$, initialization already determines the answer, as shown at the start of Appendix~\ref{app:general-depth-proofs}. The cost bound follows from $C_\tau=H(\boldsymbol\delta)$ and $P_D\ge1$. Assume $\rho_0>0$ below. On \(\mathcal E_\delta\), finite stopping and correctness follow from Proposition~\ref{prop:budgeted-finite-stopping}. Now that \(\tau\) is known to be finite, Proposition~\ref{prop:aggregate-subtree-charging} applies with \(T=\tau\) and gives
\[
C_\tau\le P_DH(\boldsymbol{\delta}).
\]
The \(\mathcal{O}(D^2)\) form follows from the final statement of Proposition~\ref{prop:aggregate-subtree-charging} for \(\alpha_h=(h+1)^2\).

By Eq.~\eqref{eq:E-delta}, each feasible allocation satisfies \(\mathbb P(\mathcal E_\delta)\ge 1-\delta\). For the optimized statement, choose a near-optimal feasible allocation \(\boldsymbol{\delta}^{\eta}\) satisfying the uniform regularity condition in Theorem~\ref{thm:budgeted-general-depth}, with
\[
H(\boldsymbol{\delta}^{\eta})\le H^*+\eta.
\]
Apply the fixed-allocation bound to the run using this allocation, on its corresponding simultaneous-validity event. This gives the claimed high-probability optimized bound. If an allocation attaining the infimum satisfies the same regularity constants, take \(\eta=0\). Otherwise, \(H^*\ge |\mathrm{Ch}(r)|\ge 2\) ensures that any \(\eta\le H^*\) gives the stated \(H^*\)-order bound.
\end{proof}

\subsection{Proofs and Qualifications for Oracle Robustness}\label{app:robustness-proofs}
The robustness results control the probability of an incorrect finite output. Extending the unbiased finite-stopping or cost theorem would require additional assumptions.

\subsubsection{Proof of Theorem~\ref{thm:biased-slow}}
Write $\theta_v=V^*(v)+b_v$. To account for the bias, replace Eq.~\eqref{eq:slow-interval} by
\[
I_v^{S,\xi}(t)=\bigcap_{j=1}^{N_v^S(t)}
[\widehat V_{v,j}-\beta_v(j,\delta_v)-\xi,
 \widehat V_{v,j}+\beta_v(j,\delta_v)+\xi],
\]
where the intersection is $(-\infty,+\infty)$ before any samples are taken. The fast intervals, minimax backups, and root stopping test are unchanged. The assumed concentration around the biased mean is
\[
\mathbb P\!\left(\exists n\ge1:
 |\widehat V_{v,n}-\theta_v|>\beta_v(n,\delta_v)\right)\le\delta_v.
\]
\begin{proof}
A union bound gives an event of probability at least $1-\delta$ on which, simultaneously for all nodes and sample counts, $|\widehat V_{v,n}-\theta_v|\le\beta_v(n,\delta_v)$. On this event,
\[
|\widehat V_{v,n}-V^*(v)|
\le\beta_v(n,\delta_v)+\xi.
\]
Every inflated slow interval therefore contains the true node value, as does every fast interval. Their intersections remain nonempty and valid, and induction through the max and min backups preserves validity without further inflation. The inconsistency exit consequently cannot occur on this event. The root stopping argument from Theorem~\ref{thm:pac-correct} now proves the claim.
\end{proof}

\paragraph{Why interval inflation is required.} Without inflation, the endpoint inequalities $L_v\le V^*(v)+\xi$ and $U_v\ge V^*(v)-\xi$ would bound the error by $\varepsilon+2\xi$ whenever the root stopping test returns a separated recommendation. But they do not ensure $L_v\le U_v$: a fast interval containing $V^*(v)$ may be disjoint from a narrow slow interval around $V^*(v)+b_v$. They therefore cannot, by themselves, guarantee correctness of the unmodified algorithm for arbitrary biased oracles. Inflation by the known bias bound preserves valid, nonempty intervals. The inflated radius need not tend to zero, however, so the unbiased finite-stopping and cost proofs do not carry over.

\subsubsection{Proof of Theorem~\ref{thm:loose-envelope}}
\begin{proof}
The enlarged fast interval contains $V^*(v)$. On $\mathcal E_\delta$, every slow interval also contains $V^*(v)$, so local intersections and minimax backups remain valid. The stopping-rule proof applies, and the inconsistency exit cannot occur on $\mathcal E_\delta$.

For the fixed-history comparison, intersecting a larger fast interval with the same slow interval gives a larger local interval. Enlarging the child intervals weakly decreases their lower endpoints and weakly increases their upper endpoints; both max and min backups preserve these inequalities. Induction up the common exposed tree, with the enlarged local intervals intersected at each node, proves the claimed containment. Actual adaptive runs may choose different queries and expansions, so this comparison does not order their realized costs.
\end{proof}

\subsection{Qualifications for the Reference Complexity and Local Regularity}\label{app:theory-qualifications}
The correctness theorem bounds the unconditional probability that the algorithm stops and returns a recommendation that violates the requested accuracy.

We also note that the half-gap condition in Assumption~\ref{ass:local-scale-regularity} restricts how the hybrid cost curve may cross the deterministic fast-bias cutoff. If
\[
\Delta_v^{\mathrm{eff}}/8<B(h(v))\le\Delta_v^{\mathrm{eff}}/4,
\]
then $\Gamma_v(\Delta_v^{\mathrm{eff}},\delta_v)=0$ but $\Gamma_v(\Delta_v^{\mathrm{eff}}/2,\delta_v)>0$, so the condition fails. Polynomial scaling of the slow-sample cost alone therefore does not establish regularity of the hybrid curve.

\end{document}